\documentclass{article}

\usepackage{microtype}
\usepackage{graphicx}
\usepackage{subfigure}

\usepackage{booktabs} % for professional tables

\usepackage{hyperref}

\usepackage{textgreek}

\usepackage[margin=1in]{geometry}
\usepackage{amsmath}
\usepackage{amssymb}
\usepackage{mathtools}
\usepackage{amsthm}
\usepackage{comment}
\usepackage[capitalize,noabbrev]{cleveref}

\theoremstyle{plain}
\newtheorem{theorem}{Theorem}[section]
\newtheorem{proposition}[theorem]{Proposition}
\newtheorem{lemma}[theorem]{Lemma}
\newtheorem{corollary}[theorem]{Corollary}
\theoremstyle{definition}
\newtheorem{definition}[theorem]{Definition}

\theoremstyle{remark}

\usepackage[textsize=tiny]{todonotes}

\begin{document}

%\twocolumn[
\title{On the Spectral Bias of Convolutional Neural Tangent and Gaussian Process Kernels}
\date{} 
\author{Amnon Geifman$^1$ ~~~ Meirav Galun$^1$ ~~~ David Jacobs$^2$~~~ Ronen Basri$^1$\\~\\
$^1$Weizmann Institute of Science, Rehovot, Israel \\
$^2$University of Maryland at College Park, MD, USA}
\maketitle




%\icmlcorrespondingauthor{Amnon Geifman}{amnon.geifman@weizmann.ac.il}
%\icmlcorrespondingauthor{Firstname1 Lastname1}{first1.last1@xxx.edu}
%\icmlcorrespondingauthor{Firstname2 Lastname2}{first2.last2@www.uk}

%\icmlkeywords{CNTK, Neural tangent kernel, Over-parameterized CNNs}

\vskip 0.3in




%\input{main_text.tex}
%\documentclass[12pt]{article}
%\usepackage[utf8]{inputenc}

%% Useful packages
%\usepackage{amsmath}
%\usepackage{amssymb}
%\usepackage{amsthm}
%\usepackage[colorinlistoftodos]{todonotes}
%\usepackage{dsfont} 

%\setlength{\columnsep}{1cm} \setlength{\textwidth}{6.75in}
%\setlength{\textheight}{8.5in} \setlength{\oddsidemargin}{-0in}
%\setlength{\evensidemargin}{-0in} \setlength{\topmargin}{-0in}
%\setlength{\headheight}{0in} \setlength{\headsep}{0.5in}
%\setlength{\footskip}{0.5in} \setlength{\columnsep}{1cm}

%\usepackage{hyperref}
%\usepackage{comment}

%\newtheorem{theorem}[]{Theorem}
%\newtheorem{claim}[]{Claim}
%\newtheorem{lemma}[]{Lemma}
%\newtheorem{proposition}[]{Proposition}
%\newtheorem{definition}{Definition}[]
%\newtheorem{corollary}{Corollary}[]
% Vectors and matrices
% Symbols
\newcommand{\Real}{\mathbb{R}}
\newcommand{\N}{\mathbb{N}}
\newcommand{\Sphere}{\mathbb{S}}
\newcommand{\Ind}{\mathds{I}}
\newcommand{\norm}[1]{\left\lVert#1\right\rVert}
\newcommand{\abs}[1]{\left\vert#1\right\rvert}

% Vectors and matrices
\newcommand{\aw}{\mathbf{a}}
\newcommand{\bias}{\mathbf{b}}
\newcommand{\f}{\mathbf{f}}
\newcommand{\g}{\mathbf{g}}
\newcommand{\h}{\mathbf{h}}
\newcommand{\ii}{\mathbf{i}}
\newcommand{\jj}{\mathbf{j}}
\newcommand{\kk}{\mathbf{k}}
\newcommand{\m}{\mathbf{m}}
\newcommand{\n}{\mathbf{n}}
\newcommand{\p}{\mathbf{p}}
\newcommand{\uu}{\mathbf{u}}
\newcommand{\vv}{\mathbf{v}}
\newcommand{\w}{\mathbf{w}}
\newcommand{\x}{\mathbf{x}}
\newcommand{\X}{\mathbf{X}}
\newcommand{\y}{\mathbf{y}}
\newcommand{\z}{\mathbf{z}}
\newcommand{\s}{\mathbf{s}}
\newcommand{\tbf}{\mathbf{t}}
\newcommand{\R}{\mathcal{R}}
\newcommand{\feqn}{f^{\mathrm{Eq}}}
\newcommand{\ftrace}{f^{\mathrm{Tr}}}
\newcommand{\fgap}{f^{\mathrm{GAP}}}
\newcommand{\kr}{\mathrlap{^-}\boldsymbol{k}}
\newcommand{\kreqn}{\boldsymbol{k}^{\mathrm{EqNet}}}
\newcommand{\krtrace}{\boldsymbol{k}^{\mathrm{Tr}}}
\newcommand{\krgap}{\boldsymbol{k}^{\mathrm{GAP}}}
\newcommand{\lambeqn}{\lambda^{\mathrm{EqNet}}}
\newcommand{\lambtr}{\lambda^{\mathrm{Tr}}}
\newcommand{\lambgap}{\lambda^{\mathrm{GAP}}}
\newcommand{\sgn}{{\mathrm{sign}}}
\newcommand{\cnt}{{\mathrm{cnt}}}
\newcommand{\Unif}{\mathrm{Unif}}

\newcommand{\zero}{\mathbf{0}}
\newcommand{\one}{\mathbf{1}}
\newcommand{\dbar}{{d}}
\newcommand{\qbar}{\zeta}
\newcommand{\ms}{\mathbb{MS}(\qbar,d)}
\newcommand{\allones}{\mathbf{1}}

\newcommand{\rb}[1]{\textcolor{blue}{[Ronen: #1]}}
\newcommand{\ag}[1]{\textcolor{red}{[Amnon: #1]}}
\newcommand{\mg}[1]{\textcolor{cyan}{[Meirav: #1]}}
\newcommand{\dwj}[1]{\textcolor{violet}{[David: #1]}}

\begin{comment}

\renewcommand{\rb}[1]{\textcolor{blue}{}}
\renewcommand{\ag}[1]{\textcolor{red}{}}
\renewcommand{\mg}[1]{\textcolor{cyan}{}}
\renewcommand{\dwj}[1]{\textcolor{violet}{}}

\end{comment}

%\begin{document}
%\title{On The Spectral Analysis of CNTK and CGPK}
%\date{}
%\maketitle

\begin{abstract}
We study the properties of various over-parametrized convolutional neural  architectures through their respective Gaussian process and neural tangent kernels. We prove that, with normalized multi-channel input and ReLU activation, the eigenfunctions of these kernels with the uniform measure are formed by products of spherical harmonics, defined over the channels of the different pixels. We next use hierarchical facotorizable kernels to bound their respective eigenvalues. We show that the eigenvalues decay polynomially, quantify the rate of decay, and derive measures that reflect the composition of hierarchical features in these networks. Our results provide concrete quantitative characterization of over-parameterized convolutional network architectures. 
\end{abstract}

%\tableofcontents
\section{Introduction}

Convolutional Neural Networks (CNNs) \cite{lecun1998gradient} have produced dramatic improvements over past machine learning approaches \cite{krizhevsky2012imagenet,simonyan2015very,he2016deep}.  
Two key properties that distinguish CNNs are their ability to encode geometric properties of the data, by incorporating multiscale analysis and invariance or equivariance. Shift invariant networks produce the same output when the input is shifted, which can be valuable, for example, in classification tasks in which objects are not well aligned.  Shift equivariant networks produce shifted output when the input is shifted, and are important in image-to-image networks that, for example, denoise or segment the input \cite{kim2016accurate,ledig2017photo,ulyanov2018deep}. Multiscale feature representations naturally arise in these networks through their depth. 
 
However, we still lack a theoretical analysis of CNNs that can quantitatively predict their behavior. Our analysis builds on the Gaussian Process and Neural Tangent kernels (resp.\ GPK and NTK).  It has been shown theoretically that massively overparameterized networks can be well approximated by a linearization about their initialization \cite{allen-zhu2019,arora2019exact,du2019gradient,jacot2018neural}.  With this linearization, neural networks become kernel regressors, with training dynamics and smoothness properties determined by the eigenfunctions and eigenvalues of their kernel, which determine their Reproducing Kernel Hilbert Space (RKHS). 

A series of interesting works has derived the spectrum of NTK for fully connected networks (denoted FC-NTK). This eigen-analysis tells us which functions a network learns most rapidly, since the speed of learning an eigenfunction with gradient descent (GD) is inversely proportional to the corresponding eigenvalue. For example, it allows us to determine that FC-NTK learns low frequency components of a function faster than high frequency components, and characterize the rates at which this happens \cite{bach2017breaking,Basri2019convergence,basri2020Nonuniform,bietti2020deep,bietti2019inductive}.
%geifman2020similarity}. 
So this eigen-analysis allows us to characterize the inductive bias of over-parametrized neural networks. Convolutional GPKs and NTKs (resp.\ CGPKs and CNTKs) have been derived for convolutional networks \cite{arora2019exact,novak2018bayesian}, but a characterization of their spectral bias  is still missing.

In this paper we investigate the Gaussian process and neural tangent kernels associated with three deep convolutional architectures. In particular, we consider kernels associated with a shift equivariant architecture, a convnet in which the last layer is fully connected, and a convnet with a final global average pooling step. The former network is applicable to various image processing tasks. The second network is similar in architecture to AlexNet and VGG \cite{krizhevsky2012imagenet,simonyan2015very}. The latter network resembles a residual network \cite{he2016deep}, without skip connections. All models we consider use ReLU activation. While we do not explicitly account for intermediate pooling, our work can readily be extended to handle such layers as well. 

We assume our networks receive multi-channel input signals, with the channels for each pixel normalized to unit norm. 
Our results establish that:
\begin{enumerate}
    \item The eigenfunctions of the three kernels include either products of spherical harmonics (SHs) or their shift invariant sums, with each harmonic term defined over the channels of one pixel.
    \item The corresponding eigenvalues decay polynomially with the frequency of the eigenfunctions, at a rate that depends on the number of input channels.
    \item The eigenvalues include a multiplicative factor that reflects the hierarchical structure of the features in the corresponding network. With the equivariant architecture, this factor is large for pixels at the center of the receptive field and small in the periphery. For the other two kernels this multiplicative factor is large for pixels close to each other and becomes very small for pixels far from each other.
\end{enumerate}

Our results show that CNNs, like FC-networks, are biased toward learning low frequency functions. However, point (2) tells us that CNNs can learn high frequency functions, when these are localized in a subset of the pixels, much more rapidly than FC networks.  It is important to keep in mind that we are referring to the frequency of the function the network has learned, which is a function over the space of all images; this does not refer to the power spectrum of individual images. Put differently, high frequency reflects high variability of the target function for similar input images. Interestingly, the rate of decay does not depend directly on image size or the size of the convolution filter. Point (3) tells us that CNNs learn spatially localized functions more rapidly than functions with global dependence on an entire image. This shows that CNNs can change their output significantly based on relatively small, spatially localized features; FC networks will take much longer to learn these spatially localized features.  In both cases we quantify this bias.

\section{Preliminaries and notations} \label{sec:preliminaries}

We consider a multi-channel 1-D input signal $\x$ of length $d$ with $\qbar$ channels, represented by a $\qbar \times d$ matrix. We further set $D=\qbar d$ and refer to the columns of $\x$ as \emph{pixels}, denoted $\x^{(i)} \in \Real^{\qbar}$, $i \in [d]$. We note that our results can readily be applied also to 2-D, multi-channel signals. We assume further that the entries of each pixel are normalized to unit norm, i.e., $\left\|\left(x_{1}^{(i)},...,x_{\qbar }^{(i)}\right)\right\|=1$. The input space, therefore, is a Cartesian product of spheres, which we call \emph{multisphere}, i.e., $\x\in\ms=\underbrace{\Sphere^{\qbar-1}\times ...\times \Sphere^{\qbar-1}}_{d} \subset \Sphere^{D-1}$ (with radius $\sqrt{d}$). We denote by $s_i\x = (x_{i+1}, ... x_d, x_1, ... x_{i})$ the cyclic shift of $\x$ to the left by $i$ pixels. 

We use multi-index notations, i.e., $\n=(n_1,...,n_d),\kk=(k_1,...,k_d)\in\N^d$ to denote vectors of polynomial orders or frequencies. $\N$ denotes the natural numbers including zero, and $b_{\n}, \lambda_{\kk} \in \Real$ are scalars that depend on vectors of indices $\n$ or $\kk$. We denote monomials by $\tbf^{\n}=t_1^{n_1}t_2^{n_2}\cdot\ldots\cdot t_{d}^{n_{d}}$ with $\tbf \in \Real^d$, and allow also for a scalar exponent, i.e., $\tbf^n=(t_1\cdot...\cdot t_d)^n$. For $\uu,\vv\in\Real^d$ we say that $\uu \ge \vv$ if $u_i \ge v_i$ for all $i \in [d]$. Therefore, the power series $\sum_{\n \ge \zero} b_\n \tbf^\n$ should read $\sum_{n_1 \ge 0,n_2 \ge 0,...} b_{n_1,n_2,...} t_1^{n_1} t_2^{n_2}...$  

We denote the uniform distribution in a domain $\Omega$ by $\Unif(\Omega)$. We write $f(x) \sim g(x)$ when $\lim_{x \rightarrow \infty} f(x)/g(x)=1$. Throughout the paper we assume all kernels are differentiable at zero infinitely many times and their power series converge in the hypercube $[-1,1]^d$. Our theorems and lemmas are proved in the appendix.

%\ag{Maybe we want to add some real preliminaries-- RKHS, kernel regression, integral operator, RKHS norm etc}\ag{define the integral operator and kernel regression}

\subsection{The network model}  \label{sec:network}

We consider convolutional neural network architectures (Figure~\ref{fig:architecture}) of the following form. Given a multi-channel 1-D input signal $\x \in \ms$ arranged in a $\qbar \times d$ matrix, we use a shift equivariant backbone and three heads to produce scalar features. The network, defined  in Table~\ref{tab:network_model}, begins with a $1 \times 1$ convolution layer, followed by $L-1 \ge 1$ stride-1 convolutional layers with filters of size $q$, producing at each layer the same number of feature channels $m$ in each of the $d$ locations.

\begin{figure}[tb]
    \centering
    \includegraphics[width=0.55\textwidth]{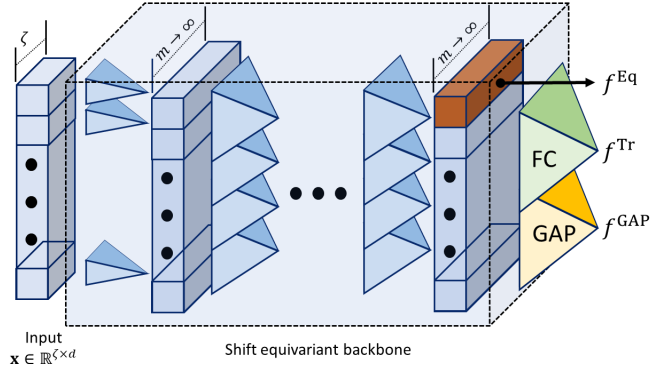}    \caption{\small Network architecture. An input signal $\x \in \Real^{\qbar \times d}$ (left column) is fed into an equivariant backbone (dashed box), producing $m$ feature channels for each pixel, first using $1 \times 1$ convolution, followed by $L-1$ convolution layers with filters of size $q$ (marked by blueish pyramids), interleaved with ReLU activation. The backbone is followed by one of three heads, $\feqn$, $\ftrace$ and $\fgap$.}
    \label{fig:architecture}
\end{figure}
\begin{table}[tb]
    \centering
    \caption{\small Network model.}
    \fbox{
    \begin{minipage}{\linewidth}
    \begin{tiny}
    \begin{enumerate}
        \item $\f^{(1)}(\x,\theta) = \sigma\left( W^{(1)}\x \right)$.
        \item $\f_i^{(l)}(\x,\theta) = \sigma\left(\sqrt{\frac{c_\sigma}{mq}} \left(\sum_{j=1}^m W_{:,i,j}^{(l)}*\f_{j}^{(l-1)} (\x,\theta) \right) \right)$,
        \item with $l \in \{2,\ldots,L\}$, $i \in [m]$, and three heads:
        \begin{enumerate}
            \item $\feqn(\x,\theta) = \langle \w^{\mathrm{Eq}},\f^{(L)}_{:,1}(\x,\theta)\rangle$.
            \item $\ftrace(\x,\theta) = \frac{1}{\sqrt{d}} \langle W^{\mathrm{Tr}}, \f^{(L)}(\x,\theta) \rangle$.
            \item $\fgap(\x,\theta) = \frac{1}{d} \w^{\mathrm{GAP}} \f^{(L)}(\x,\theta) \one$.
        \end{enumerate}
    \end{enumerate}        
    Here $\f^{(l)} \in \Real^{m \times d}$ ($ \l \in [L]$); $\theta=\left(W^{(1)},...,W^{(L)}, W^{\mathrm{Tr}}, \w^{\mathrm{GAP}}, \w^\mathrm{Eq}\right)$ are learnable parameters initialized with ${\cal N}(0,I)$. $W^{(1)} \in \Real^{m \times \qbar}$, $W^{(l)} \in \Real^{q \times m \times m}$ (i.e., $W_{:,i,j}^{(l)}$ is a filter of size $q$), $W^{\mathrm{Tr}} \in \Real^{m \times d}$ ($\langle \cdot,\cdot \rangle$ denotes the standard inner product between matrices), $\w^\mathrm{Eq},\w^{\mathrm{GAP}} \in \Real^{1 \times m}$ and $\one = (1,...,1)^T \in \Real^{d}$. '*' denotes cyclic convolution; $\sigma$ is the ReLU function, and for ReLU, $c_{\sigma} = 1/\left( \mathbb{E}_{z \sim \mathcal{N}(0,1)} [\sigma(z)^2] \right)=2$.
    \end{tiny}
    \end{minipage}}
    \label{tab:network_model}
\end{table}

The $\feqn$ head produces one scalar entry for the shift equivariant network (i.e., the tuple $(f^{\mathrm{Eq}}(\x,\cdot),...,f^{\mathrm{Eq}}(s_{d-1}\x,\cdot))$ produces the shift-equivariant output); $\ftrace$ corresponds to applying a fully connected layer at the top layer, and $\fgap$ corresponds to applying global average pooling, resulting in a shift invariant network. With each of the heads, the network parameters are trained for regression with the mean square error (MSE) loss.

\subsection{Derivation of CGPK and CNTK} \label{sec:CNTK_formula}

Previous work showed that in the limit of infinite width and with small initialization, massively over-parameterized neural networks become kernel regressors with a family of kernels called neural tangent kernels~\cite{jacot2018neural}. Let $f(\x,\theta)$ denote a network, then for a pair of inputs $\x_i,\x_j$, the corresponding NTK is defined as $\kr(\x_i,\x_j) = \mathbb{E}_{\theta \sim \mathcal{P}} \left< \frac{\partial f(\x_i,\theta)}{\partial \theta}, \frac{\partial f(\x_j,\theta)}{\partial \theta} \right>$. A related kernel, called the Gaussian process (or random feature) kernel, is obtained if the weights are kept in their initial random values, and only the last layer of the network is optimized in training \cite{cho2009kernel}.

\cite{arora2019exact} derived expressions for CNTK and CGPK for convolutional networks. The formulas in Table~\ref{tab:cgpk_cntk} adapt these expressions to our convolutional architectures and to multisphere inputs. We note that with general input in $\Real^D$ additional normalization steps are needed. We refer the reader to \cite{arora2019exact} for the full derivation.

\begin{table}[tb]
    \centering    
    \caption{\small CGPK and CNTK formulas.}
    \fbox{
    \begin{minipage}{\linewidth}
    \begin{tiny}
    Given $\x,\z\in \ms$, we denote respectively by $X$ and $Z$ their $\qbar \times d$ matrix representations. Let $\Sigma^{(0)}(\x,\z) = \Theta^{(0)}(\x,\z) = X^TZ$. For $l\in [L]$ and for $ i,j \in [d] $,
    \begin{enumerate}
        \item $\Sigma_{i,j}^{(1)}(\x,\z) = \kappa_1\left(\Sigma_{i,j}^{(0)}(\x,\z) \right)$.
        \item $\dot \Sigma_{i,j}^{(1)}(\x,\z) = \kappa_0\left(\Sigma_{i,j}^{(0)}(\x,\z) \right)$.
        \item $\Theta_{i,j}^{(l)}(\x,\z) = \frac{1}{2q} \sum_{r=0}^{q-1} \left[ \tilde{ \dot \Sigma}_{i+r,j+r}^{(l)}(\x,\z) \tilde\Theta_{i+r,j+r}^{(l-1)}(\x,\z) + \right.$\\ 
        $~~~~~~~~~~~~~~~~~~~~~~~~~~~~~~~~~~~~~~~~~~~~~~~~~~
        \left. \tilde\Sigma_{i+r,j+r}^{(l)}(\x,\z)\right]$.
        \item $\Sigma_{i,j}^{(l+1)}(\x,\z) = \kappa_1\left(\frac{1}{q} \sum_{r=0}^{q-1}  \tilde\Sigma_{i+r,j+r}^{(l)}(\x,\z) \right)$.
        \item $\dot \Sigma_{i,j}^{(l+1)}(\x,\z) = \kappa_0\left(\frac{1}{q} \sum_{r=0}^{q-1}  \tilde\Sigma_{i+r,j+r}^{(l)}(\x,\z) \right)$,
    \end{enumerate}
    where we denote by $\tilde\Sigma$, $\tilde{\dot\Sigma}$, and $\tilde\Theta$ respectively $\Sigma$, $\dot\Sigma$ and $\Theta$ whose rows and columns are extended with circular padding. Additionally, with ReLU activation
    \begin{enumerate}
        \item $\kappa_0(u)=\frac{\pi - \arccos(u)}{\pi}$.
        \item $\kappa_1(u)=\frac{(\pi - \arccos(u))u+\sqrt{1-u^2}}{\pi}, ~~ u\in[-1,1]$.
    \end{enumerate}
    \end{tiny}
    \end{minipage}
    }
    \label{tab:cgpk_cntk}

    \centering
    \caption{\small Kernel heads.}
    \fbox{
    \begin{minipage}{\linewidth}
    \begin{tiny}
    \begin{enumerate}
        \item CGPK-EqNet (corresponds to $\feqn$): $\Sigma_{1,1}^{(L)}(\x,\z)$.
        \item CGPK (corresponds to $\ftrace$): $\frac{1}{d} \sum_{i=1}^{d} \Sigma_{i,i}^{(L)}(\x,\z)$.
        \item CGPK-GAP (corresponds to $\fgap$): $\frac{1}{d^2} \sum_{i=1}^{d} \sum_{j=1}^{d} \Sigma_{i,j}^{(L)}(\x,\z)$.
        \item CNTK-EqNet (corresponds to $\feqn$): $\Theta_{1,1}^{(L)}(\x,\z)$.
        \item CNTK (corresponds to $\ftrace$): $\frac{1}{d} \sum_{i=1}^{d} \Theta_{i,i}^{(L)}(\x,\z)$.
        \item CNTK-GAP (corresponds to $\fgap$): $\frac{1}{d^2} \sum_{i=1}^{d} \sum_{j=1}^{d} \Theta_{i,j}^{(L)}(\x,\z)$.
    \end{enumerate}
    \end{tiny}
    \end{minipage}
    }
    \label{tab:kernels}
\end{table}

\begin{comment}
Given $\x,\z\in \ms$, we denote respectively by $X$ and $Z$ their $\qbar \times d$ matrix representations. Let
\begin{equation*}
    \Sigma^{(0)}(\x,\z) = \Theta^{(0)}(\x,\z) = X^TZ.
\end{equation*}
Next, for $l\in [L]$ and for $ i,j \in [d] $,
\begin{align*}
    \Sigma_{i,j}^{(1)}(\x,\z) &= \kappa_1\left(\Sigma_{i,j}^{(0)}(\x,\z) \right)\\
    \dot \Sigma_{i,j}^{(1)}(\x,\z) &= \kappa_0\left(\Sigma_{i,j}^{(0)}(\x,\z) \right)\\
    \Theta_{i,j}^{(l)}(\x,\z) &=\frac{1}{2q}\sum_{r=0}^{q-1}\left[ \tilde{ \dot \Sigma}_{i+r,j+r}^{(l)}(\x,\z) \tilde\Theta_{i+r,j+r}^{(l-1)}(\x,\z) \right. \\ &+ \left. \tilde\Sigma_{i+r,j+r}^{(l)}(\x,\z)\right]\\
    \Sigma_{i,j}^{(l+1)}(\x,\z) &= \kappa_1\left(\frac{1}{q} \sum_{r=0}^{q-1}  \tilde\Sigma_{i+r,j+r}^{(l)}(\x,\z) \right)\\
    \dot \Sigma_{i,j}^{(l+1)}(\x,\z) &= \kappa_0\left(\frac{1}{q} \sum_{r=0}^{q-1}  \tilde\Sigma_{i+r,j+r}^{(l)}(\x,\z) \right),
\end{align*}
where we denote by $\tilde\Sigma$ and $\tilde\Theta$ respectively $\Sigma$ and $\Theta$ whose rows and columns are extended with circular padding. Additionally, with ReLU activation
\begin{align*}
    \kappa_0(u)&=\frac{\pi - \arccos(u)}{\pi}, \\
    \kappa_1(u)&=\frac{(\pi - \arccos(u))u+\sqrt{1-u^2}}{\pi}, ~~ u\in[-1,1].
\end{align*}
\end{comment}

Note that for a pair of inputs $\x,\z$, this definition produces two matrices of kernels, $\Sigma^{(L)}_{i,j}(\x,\z)$ and $\Theta^{(L)}_{i,j}(\x,\z)$, $i,j \in [d]$, and that $\Sigma_{i,j}^{(L)}(\x,\z) = \Sigma_{1,1}^{(L)}(s_{i-1}\x,s_{j-1}\z)$ (and similarly for $\Theta$).
With these definitions, we produce six different kernels in Table~\ref{tab:kernels}--these describe three different architectures for each of the Gaussian process neural tangent kernels.

\begin{comment}
\begin{itemize}
    \item CGPK-EqNet, which corresponds to $\feqn$ is defined by $\Sigma_{1,1}^{(L)}(\x,\z)$.
    \item CGPK, which corresponds to $\ftrace$, is defined by $\frac{1}{d} \sum_{i=1}^{d} \Sigma_{i,i}^{(L)}(\x,\z)$.
    \item CGPK-GAP, which corresponds to $\fgap$ is defined by $\frac{1}{d^2} \sum_{i=1}^{d} \sum_{j=1}^{d} \Sigma_{i,j}^{(L)}(\x,\z)$.
    \item CNTK-EqNet, which corresponds to $\feqn$ is defined by $\Theta_{1,1}^{(L)}(\x,\z)$.
    \item CNTK, which corresponds to $\ftrace$, is defined by $\frac{1}{d} \sum_{i=1}^{d} \Theta_{i,i}^{(L)}(\x,\z)$.
    \item CNTK-GAP, which corresponds to $\fgap$ is defined by $\frac{1}{d^2} \sum_{i=1}^{d} \sum_{j=1}^{d} \Theta_{i,j}^{(L)}(\x,\z)$.
\end{itemize}
\end{comment}

Below we refer to these six kernels as CGPKs and CNTKs. Note that CGPK-EqNet and CNTK-EqNet produces a single output, corresponding to the first output of the equivariant network. The tuple $\left(\kr(s_0\x,s_0\z),...,\kr(s_{d-1}\x,s_{d-1}\z)\right) \in \Real^d$ (with $\kr$ either CGPK-EqNet or CNTK-EqNet) produces the full response of the equivariant network.

\section{The RKHS of CGPKs and CNTKs} \label{sec:rkhs}

Our objective is to derive the spectrum of Gaussian process and neural tangent kernels associated with convolutional networks. We do so by forming bounds using products of kernels that apply to individual pixels and are composed hierarchically.
We proceed below as follows. We first prove general results for kernels that are functions of inner products between pixels (Sec.~\ref{sec:multidot}). We next consider factorizable kernels and derive their spectrum (Sec.~\ref{sec:factorizable}). Then, in Sec.~\ref{sec:spatial}, we refine our expressions to account for hierarchical kernels. We finally use these derivations in Sec.~\ref{sec:equivariant} and~\ref{sec:trace} to prove bounds for all six kernels.

\subsection{Multi-dot product kernels} \label{sec:multidot}

It can be readily shown that CNTK and CGPK associated with the shift equivariant network are functions of dot products of corresponding pixels. We refer to such kernels as \emph{multi-dot product} and prove in Appendix~\ref{app:multidot} several results regarding their spectral properties, which we briefly summarize here.

We call a kernel $\kr(\x,\z):\ms\times\ms$ $\rightarrow \Real$ \emph{multi-dot product} if $\kr(\x,\z) = \kr(\tbf)$, where $\tbf=\left(\langle \x^{(1)},\z^{(1)}\rangle,..,\langle \x^{({\dbar})},\z^{({\dbar})}\rangle\right) \in [-1,1]^d$. (Note the overload of notation, which should be clear by context.) Multi-dot product kernels can be written via Mercer's decomposition as
\begin{align} \label{eq:mercer}
    \kr(\x,\z)=\sum_{\kk,\jj} \lambda_{\kk}Y_{\kk,\jj}(\x)Y_{\kk,\jj}(\z),
\end{align}
where $Y_{\kk,\jj}(\x)$ ($\kk, \jj \in \mathbb{N}^d$), the eigenfunctions of $\kr$, are products of spherical harmonics in $\Sphere^{\qbar-1}$, $\prod_{r=i}^{d} Y_{k_ij_i}\left(\x^{(i)}\right)$, with $k_i \ge 0$, $j_i \in [N(\qbar,k_i)]$, and $N(\qbar,k_i)$ denotes the number of harmonics of frequency $k_i$ in $\Sphere^{\qbar-1}$. Such products are harmonic polynomials in $\Sphere^{D-1}$. With $\qbar=2$, these are products of Fourier series in a multi-dimensional torus. The eigenvalues $\lambda_\kk$ depend on the vector of frequencies, $\kk$, and are independent of the phases $\jj$. We note that a multi-dot product kernel is universal for $\ms$ if all its eigenvalues are strictly positive. Non-universal kernels are obtained, e.g., when the eigenfunctions  involve pixels outside of the receptive field of the respective network, in which case these eigenfunctions lie in the null space of the kernel.

Below we consider multi-dot product kernels that can be expressed using a multivariate power series of the form
\begin{align} \label{eq:k_taylor}
    \kr(\tbf)=\sum_{\n \ge \zero} b_\n \tbf^\n = \sum_{\n \ge \zero} b_\n \prod_{i=1}^d \langle \x^{(i)},\z^{(i)}\rangle ^{n_i}.
\end{align}
with $b_{\n}\geq 0$ for all $\n \ge \zero$. Indeed, it can be readily shown that all our CGPKs and CNTKs are positive semidefinite (PSD) and their power series coefficients are non-negative; the kernels are obtained from the univariate PSD $\kappa_0$ and $\kappa_1$ (defined in Table~\ref{tab:kernels}), whose coefficients are non-negative by sequences of multiplication, addition and composition, resulting in PSD kernels with non-negative coefficients.

It is possible to calculate the eigenvalues of multi-dot product kernels from their power series coefficients. This is established in the following lemma, which extends a result by \cite{azevedo2015eigenvalues} to multi-dot product kernels.

\begin{lemma}\label{lemma:taylor_to_eigs}
Let $\kr$ be a multi-dot product kernel with the power series given in \eqref{eq:k_taylor}, where $\x^{(i)},\z^{(i)} \in \Sphere^{\qbar-1}$ respectively are pixels in $\x,\z$.
Then, the eigenvalues $\lambda_{\kk}(\kr)$ of $\kr$ are given by 
\begin{align*}
    \lambda_{\kk}(\kr) = \left|\Sphere^{\qbar-2}\right|^d \sum_{\s \ge 0} b_{\kk+2\s} \prod_{i=1}^{\dbar}\lambda_{k_i}(t^{k_i+2s_i}),
\end{align*}
where $|\Sphere^{\qbar-2}|$ is the surface area of $\Sphere^{\qbar-2}$, and $\lambda_k(t^n)$ is the $k$'th eigenvalue of $t^n$, given by 
\begin{align*}
\lambda_{k}(t^n) = \frac{n!}{(n-k)!2^{k+1}} \frac{\Gamma \left(\frac{\qbar-1}{2}\right)\Gamma\left(\frac{n-k+1}{2}\right)}{\Gamma\left(\frac{n-k+\qbar}{2}\right)}
\end{align*}
if $n-k$ is even and non-negative, while $\lambda_k(t^n)=0$ otherwise, and $\Gamma$ is the Gamma function.
\end{lemma}

\subsection{Factorizable Kernels} \label{sec:factorizable}

Next, we consider multivariate kernels that factor into products of dot product kernels,
\begin{align*}
\kr(\x,\z) 
%&= \kappa(\langle \x^{(1)},\z^{(1)}\rangle,..,\langle \x^{(d)},\z^{(d)}\rangle)\\
    = \prod_{i=1}^d \kr_i(\langle \x^{(i)},\z^{(i)}\rangle).
\end{align*}
The power series of $\kr$ can be written as $\kr(\tbf)=\prod_{i=1}^d \left( \sum_{n=0}^\infty b_n^{(i)} t_i^n \right)$. Their eigenvalues satisfy $\lambda_{\kk}(\kr) = \prod_{i=1}^{d} \lambda_{k_i}(\kr_i)$ and can be calculated using Lemma~\ref{lemma:taylor_to_eigs}.
%, as follows:
%\begin{align*}
%    \lambda_{\kk}(\kr) &= \prod_{i=1}^{d} \lambda_{k_i}(\kr_i) \\
%    =& \left|\Sphere^{\qbar-2}\right|^d \prod_{i=1}^{d} \left( \sum_{s_i=0}^\infty b_{k_i+2s_i}^{(i)} \lambda_{k_i}\left(t^{k_i+2s_i}\right) \right).
%\end{align*}
Below we are interested specifically in kernels whose power series decay polynomially with frequency. The next theorem shows that for such kernels, the eigenvalues too decay polynomially and derives their exact decay rate.
For the theorem we further use the concept of a receptive field, which captures for a multivariate kernel the subset of variables it depends on. That is, the receptive field $\R \subseteq [d]$ is the set of indices $i$ for which there exists $\n=(...,n_i,...)$ with $n_i \ge 1$ and $b_\n \ne 0$. This condition ensures that there exists a term in the power series expansion that depends on pixel $i$. 
%In a network the receptive field associated with an (intermediate or output) node $v$ includes all the input nodes for which there is a path leading to $v$.

\begin{theorem}
\label{thm:taylor_to_eigs}
Let $\kr$ be a factorizable multi-dot product kernel with inputs $\x,\z\in \ms$, and let $\R \subseteq [d]$ denote its receptive field. Suppose that $\kr$ can be written as a multivariate power series, $\kr(\tbf)=\sum_{\n \ge 0} b_\n \tbf^\n$ with
\begin{align*}
    b_\n \sim %c \n^{-\nu} = 
    c \prod_{i \in \R, \, n_i > 0} n_i^{-\nu}.
\end{align*}
with constants $c>0$, non-integer $\nu > 1$, and $b_\n=0$ if $n_i>0$ for any $i \not\in \R$. Then the eigenfunctions of $\kr$ w.r.t the uniform measure are the SH-products. Moreover, let $\kk \in \N^d$ be a vector of frequencies. Then, the eigenvalues $\lambda_\kk(\kr)$ satisfy
\begin{align*}
    \lambda_\kk \sim 
    %c \kk^{-(\qbar+2\nu-3)} = 
    \tilde c \prod_{i \in \R, \, k_i>0}  k_i^{-(\qbar+2\nu-3)}.
\end{align*}
Finally, $\lambda_\kk=0$ if $k_i>0$ for any $i \not\in \R$.
\end{theorem}

%\rb{The proof combines results by \cite{flajolet2009analytic} and \cite{bietti2020deep}.}
The theorem above improves over previous results by \cite{azevedo2015eigenvalues,scetbon2021spectral} in various ways. Specifically, these authors considered only dot product kernels (i.e., $d=1$) and bounded the decay of their eigenvalues by a non-tight upper bound. Here we provide tight upper and lower bounds, which are identical up to a constant factor, for any $d \ge 1$.
%The exponent in their bound is half of ours

\subsection{Positional bias of eigenvalues} \label{sec:spatial}

CNNs process signals by producing a hierarchy of learned features that emerge via repeated convolutions, interleaved with non-linear activations. It is natural to ask, therefore, to what extent this hierarchy is reflected in the RKHS of the CGPKs and CNTKs. 
To answer this question we consider kernels formed by hierarchical composition of kernels. We are further interested in such kernels that are both factorizable and whose power series decay polynomially. For such kernels we prove that their eigenvalues are larger for eigenfunctions that depend on pixels near the center of their receptive field and smaller for eigenfunctions that depend only on peripheral pixels. This in turn will allow us to derive similar bounds for the kernels associated with the equivariant network. For the trace and GAP kernels this will reveal bias to eigenfunctions that depend on nearby pixels, compared to those that depend only on more distant pixels.

\begin{definition}
A kernel $\kr^{(L)}: [-1,1]^{d} \rightarrow \Real$ is called (stride-1) hierarchical of depth $L>1$ and filter size $q$ if there exists a sequence of kernels $\kr^{(1)},...,\kr^{(L)}$ such that $\kr^{(l)}(\tbf) = f^{(l)}\left(\kr^{(l-1)}(s_0\tbf),...,\kr^{(l-1)}(s_{q-1}\tbf)\right)$ with $f^{(l)}:\Real^q \rightarrow \Real$ and $\kr^{(1)}(\tbf)=f^{(1)}(t_1)$, $t_1 \in [-1,1]$. 
\end{definition}

Similar to feed-forward networks, hierarchical kernels induce a tree structure in which the leaf nodes represent the variables $t_1,...,t_d$ and nodes represent function applications $f^{(l)}$. Each variable may be connected by multiple paths to the root node, which corresponds to the final kernel $\kr^{(L)}$. The number of paths from a node $t_i$ to the root, denoted $p_i^{(L)}$, plays an important role in the magnitude of the eigenvalues. In the next theorem we bound both the power series coefficients and the eigenvalues of hierarchical kernels, showing that the bounds depend on the number of paths in the hierarchical tree, which in turn is determined by the position of the pixel within the receptive field.

\begin{theorem} \label{thm:hierarchy}
Let $\kr^{(L)}=\sum_{\n \ge 0} b_\n \tbf^\n$ be hierarchical and factorizable of depth $L>1$ with filter size $q$, so that $b_\n = c \prod_{i=1,n_i>0}^d n_i^{-\nu}$ for non-integer $\nu>1$. Then there exists a scalar $A > 1$ such that:
\begin{enumerate}
    \item The power series coefficients of $\kr^{(L)}$ satisfy
    \begin{align*}
    b_\n \ge c_L \prod_{\substack{i=1\\n_i>0}}^d  {A}^{\min(p_i^{(L)},n_i)} n_i^{-\nu}.
    \end{align*}
    \item The eigenvalues $\lambda_\kk(\kr^{(L)})$ satisfy
    \begin{align*}
    \lambda_\kk \ge \tilde c_L \prod_{\substack{i=1\\k_i>0}}^d  {A}  ^{\min(p_i^{(L)},k_i)} k_i^{-(\qbar+2\nu-3)}
    \end{align*}
\end{enumerate}
where $c_L,\tilde c_L$ are constants that depends on $L$, and $p^{(L)}_i$ denotes the number of paths from pixel $i$ to the output of $\kr^{(L)}$.
\end{theorem}

\begin{comment}

\begin{theorem} \label{thm:hierarchy}
Let $\kr^{(L)}=\sum_{\n \ge 0} b_\n \tbf^\n$ be hierarchical and factorizable of depth $L>1$ with filter size $q$, so that $b_\n = c \prod_{i=1,n_i>0}^d n_i^{-\nu}$ for non-integer $\nu>1$. Then there exist scalar $B_\nu \ge A_\nu > 1$ (that depend on $\nu$) such that:
\begin{enumerate}
    \item The power series coefficients of $\kr^{(L)}$ satisfy
    \begin{align*}
    c_L A_\nu^{\bar p_\n^{(L)}} \prod_{\substack{i=1\\n_i>0}}^d n_i^{-\nu} \leq b_\n \leq c_L B_\nu^{\bar p_\n^{(L)}} \prod_{\substack{i=1\\n_i>0}}^d n_i^{-\nu},
    \end{align*}
    where $c_L$ is a constant that depends on $L$, $\bar p^{(L)}_{\n}=\sum_{i=1}^d \sgn(n_i) (p^{(L)}_i-1)$, and $p^{(L)}_i$ denotes the number of paths from pixel $i$ to the output of $\kr^{(L)}$.
    \item The eigenvalues $\lambda_\kk(\kr^{(L)})$ satisfy
    \begin{align*}
     c_{A_\nu,\sgn(\kk)}  \prod_{\substack{i=1\\n_i>0}}^d k_i^{-(\qbar+2\nu-3)} \le \lambda_\kk \le c_{B_\nu,\sgn(\kk)}  \prod_{\substack{i=1\\n_i>0}}^d k_i^{-(\qbar+2\nu-3)}
    \end{align*}
    where $c_{A_\nu,\sgn(\kk)} = c_L \prod_{i=1}^d \alpha_i A_\nu^{p_i^{(L)}}$, with
    \begin{align*}
    \alpha_i &= \begin{cases}
    1- \tilde c\left(1 - \frac{1}{A_\nu^{p_i^{(L)}}}\right) & k_i=0\\ 
    1 & k_i \ge 1,
    \end{cases}
    \end{align*}
    and $\tilde c$ is constant. $c_{B,\sgn(\kk)}$ is defined similarly by replacing $A_\nu$ with $B_\nu$.
\end{enumerate}
\end{theorem}

\end{comment}

The proof exploits relations between hierarchical stride-1 and stride-q kernels and relies on recursively combining power series, which are shown to maintain their polynomial decay.
We note that although the number of paths can grow rapidly with depth, the kernels we consider are normalized so that they always produce values in $[-1,1]$. As a result, both the power series coefficients and the eigenvalues are bounded in $[0,1]$. This normalization is reflected in the magnitudes of $c_L$ and $\tilde c_L$.

%To further understand Theorem \ref{thm:hierarchy}, consider two frequency vectors, $\kk$ and $\kk'$, that are identical except in two coordinates $i$ and $j$, with $k_i=0$, $k_j=\bar k$ and $k'_i=\bar k$ and $k'_j=0$. In this case, the terms in the ratio $\lambda_{\kk}/\lambda_{\kk'}$ will all cancel except $\alpha_i/\alpha_j$. Suppose now that $p_i^{(L)} \gg p_j^{(L)}$. The ratio will then tend to $1-\tilde c < 1$, which implies that $\lambda_\kk < \lambda_{\kk'}$. In other words, excluding a variable at the center of the receptive field (for which $p_i^{(L)}$ is high) reduces the eigenvalue compared to excluding a variable at the periphery of the receptive field.

In fact, for a hierarchical, stride-1 kernel of depth $L$, the (normalized) number of paths $p_i^{(L)}$ of a pixel, $\x^{(i)}$, decays exponentially with its distance from the center of the receptive field, $|i-i_c|$. The number of paths is determined by convolving a rectangular function of width $q$, the size of the convolution filter, (i.e., $r(x)=1/q$ if $0 \le x \le q$ and 0 otherwise) with itself $L$ times. Note that such a repeated convolution is equal to the density function obtained by summing $L$ random variables distributed uniformly, resulting in the Irwin–Hall distribution, whose support is stretched by the size of the convolution filter $q$. With sufficiently large $L$, using the central limit theorem and assuming the number of pixels $d$ is greater than the receptive field size, the number of paths $p_i^{(L)}$ approaches a Gaussian centered at the central pixel with variance $V_{q,L} \approx Lq^2/12$, i.e., $p_i^{(L)} \propto 
%(1/\sqrt{2\pi V_{q,L}})
\exp(-(i-i_c)^2/(2V_{q,L}))$ \cite{luo2016understanding}. Roughly speaking, we can conclude that the position of a pixel within the receptive field can have an exponential effect on both the power series coefficients and the eigenvalues of the kernel. 

Note that in the analysis above we assumed that the number of pixels $d$ is larger than the receptive field size. With a receptive field of size $d$ and cyclic convolutions, the number of paths from all pixels approaches a constant as $L$ grows. In practice, real CNNs use zero padding, which effectively ensures that $d$ is always larger than the receptive field.

\subsection{Kernels associated with the equivariant network} \label{sec:equivariant}

The following theorem characterizes the spectrums of CGPK-EqNet and CNTK-EqNet, by proving lower and upper bounds on both their power series coefficients as well as their eigenvalues.

\begin{comment}

\begin{theorem}\label{thm:cgpk-eqnet-eigs}
Let $\kr^{(L)}$ denote CGPK-EqNet of depth $L$ whose input includes $\qbar$ channels, with receptive field $\R$ and with ReLU activation. Then,
\begin{enumerate}
    \item $\kr^{(L)}$ can be written as a power series, $\kr^{(L)}(\tbf)=\sum_{\n \ge 0} b_\n \tbf^\n$ with
    \begin{align*}
    c_1 \prod_{i \in \R, n_i>0} n_i^{-2.5} \leq b_\n \leq c_2 \prod_{i \in \R, n_i>0} n_i^{-\left(1 + \frac{3}{2d}\right)}.
    \end{align*} 
    \item The eigenvalues of $\kr^{(L)}$ are bounded by
    \begin{align*}
        c_3 \prod_{i \in \R, k_i>0}  k_i^{-(\qbar+2)}\leq \lambda_\kk\leq c_4 \prod_{i \in \R, k_i>0} k_i^{-\left(\qbar+\frac{3}{d}-1 \right)},
    \end{align*}
\end{enumerate}
where $c_1,c_2,c_3,c_4$ are constants.
%The coefficients $c_1,c_2$ (resp.~$c_3,c_4$) are constants that depend on $\sgn(\n)$ (resp.~$\sgn(\kk)$), and they equal zero if $\n$ (resp.~$\kk$) includes non-zero values outside of the receptive field of $\kr^{(L)}$.
\end{theorem}

\end{comment}

\begin{theorem}\label{thm:eqnet-eigs}
Let $\kr^{(L)}$ denote either CGPK-EqNet or CNTK-EqNet of depth $L$ whose input includes $\qbar$ channels, with receptive field $\R$ and with ReLU activation. Then,
\begin{enumerate}
    \item $\kr^{(L)}$ can be written as a power series, $\kr^{(L)}(\tbf)=\sum_{\n \ge 0} b_\n \tbf^\n$ with
    \begin{align*}
     c_1 \prod_{i \in \R, n_i>0}\tilde A^{\min(p_i^{(L)},n_i)}n_i^{-\nu_a} \leq b_\n \leq c_2 \prod_{i \in \R, n_i>0} n_i^{-\nu_b},
    \end{align*} 
    
    \item The eigenvalues of $\kr^{(L)}$ are bounded by
    \begin{align*}
        c_3 \prod_{\substack{i \in \R\\k_i>0}} \tilde A^{\min(p_i^{(L)},k_i)} k_i^{-(\qbar+2\nu_a-3)}\leq \lambda_\kk \leq c_4 \prod_{\substack{i \in \R\\k_i>0}} k_i^{-\left(\qbar+2\nu_b-3 \right)},
    \end{align*}
\end{enumerate}
where for CGPK-EqNet $\nu_a=2.5$ and $\nu_b=1+3/(2d)$, while for CNTK-EqNet $\nu_a=2.5$ and $\nu_b=1+1/(2d)$.
Also, $p^{(L)}_i$ denotes the number of paths from pixel $i$ to the output of $\kr^{(L)}$, $\tilde A>1$ and $c_1,c_2,c_3$ and $c_4$ depend on $L$. 
\end{theorem}

The proof relies on the recursive formulation of the kernels, which gives rise to a multivariate version of the Fa\'{a} di Bruno formula \cite{schumann2019multivariate}, involving Bell polynomials; those are exploited to bound the coefficients of the power series expansion of the kernels.
The implication of Theorem \ref{thm:eqnet-eigs} %Theorems \ref{thm:cgpk-eqnet-eigs} and \ref{thm:cntk-eqnet-eigs} 
is that CGPK-EqNet and CNTK-EqNet are bounded from above and below by factorizable kernels. The eigenvalues of these bounding kernels decay polynomially with the product of the frequency in each pixel, $k_i$, with an exponent that depends on $\qbar$, the number of input channels. While this implies in general that with GD, low frequencies are learned faster than high frequencies, the theorem also states that it should be faster to learn target functions that vary (i.e., involve high frequencies) in a small set of pixels compared to ones that vary in many pixels. For example, according to the lower bound in Thm.~\ref{thm:eqnet-eigs}, learning an SH-product of frequency $k$ in one pixel and constant frequencies in the other pixels should require $O(k^{\qbar+2})$ GD iterations. Learning an SH-product of frequency $k/{\bar m}$ in each of $\bar m$ pixels should require $O((k/\bar m)^{\bar m(\qbar+2)})$ iterations. With $k \gg \bar m$, this is exponentially slower by a factor of $\bar m$. With $\bar m=d$, the speed of learning decays with an exponent that depends on the size of the entire signal. For such functions, over-parameterized CNNs behave similarly to fully connected networks, for which the eigenvalues of their respective kernels, FC-GPK and FC-NTK, decay roughly as $k^{-O(\tilde d)}$ for input in $\Sphere^{\tilde d-1}$. Note that the exponent in the polynomial decay does not depend on $q$, the size of the convolution filter. We conclude that over-parameterized CNNs can more efficiently learn target functions whose variation is restricted to subsets of the pixels. This is not true for fully connected networks.

While the form of the polynomial decay induces bias toward learning functions that depend on fewer pixels, according to the lower bound the multiplicative factors of these polynomials exhibit bias toward learning functions in which the variation is localized near the center of the receptive field. This bias is affected by the number of paths from a pixel node to the output and depends on the depth and the size of the convolution filter $q$.

Figure \ref{fig:coefs-bounds} provides a visualization of the power series coefficients of CGPK-EqNet and CNTK-EqNet for $d=2$ and $q=2$. The upper and lower bounds, plotted as planes in these log-log plots, indicate the maximal and minimal asymptotic directions determined by the exponents in Theorem~\ref{thm:eqnet-eigs}.
%Theorems~\ref{thm:cgpk-eqnet-eigs} and~\ref{thm:cntk-eqnet-eigs}. 
The coefficients indeed appear to lie between the bounds, with coefficients of orders that vary simultaneously along a single axis (i.e., $(n,0)$) lying close to the lower bound, while those of orders that vary along both axes (i.e., $(n,n)$) lying close to the upper bound. In this example the number of paths are equal for both pixels, so position does not affect the coefficients.

\begin{figure}[tb]\label{fig:power_coefs}
    \centering
    \includegraphics[width=0.4\textwidth]{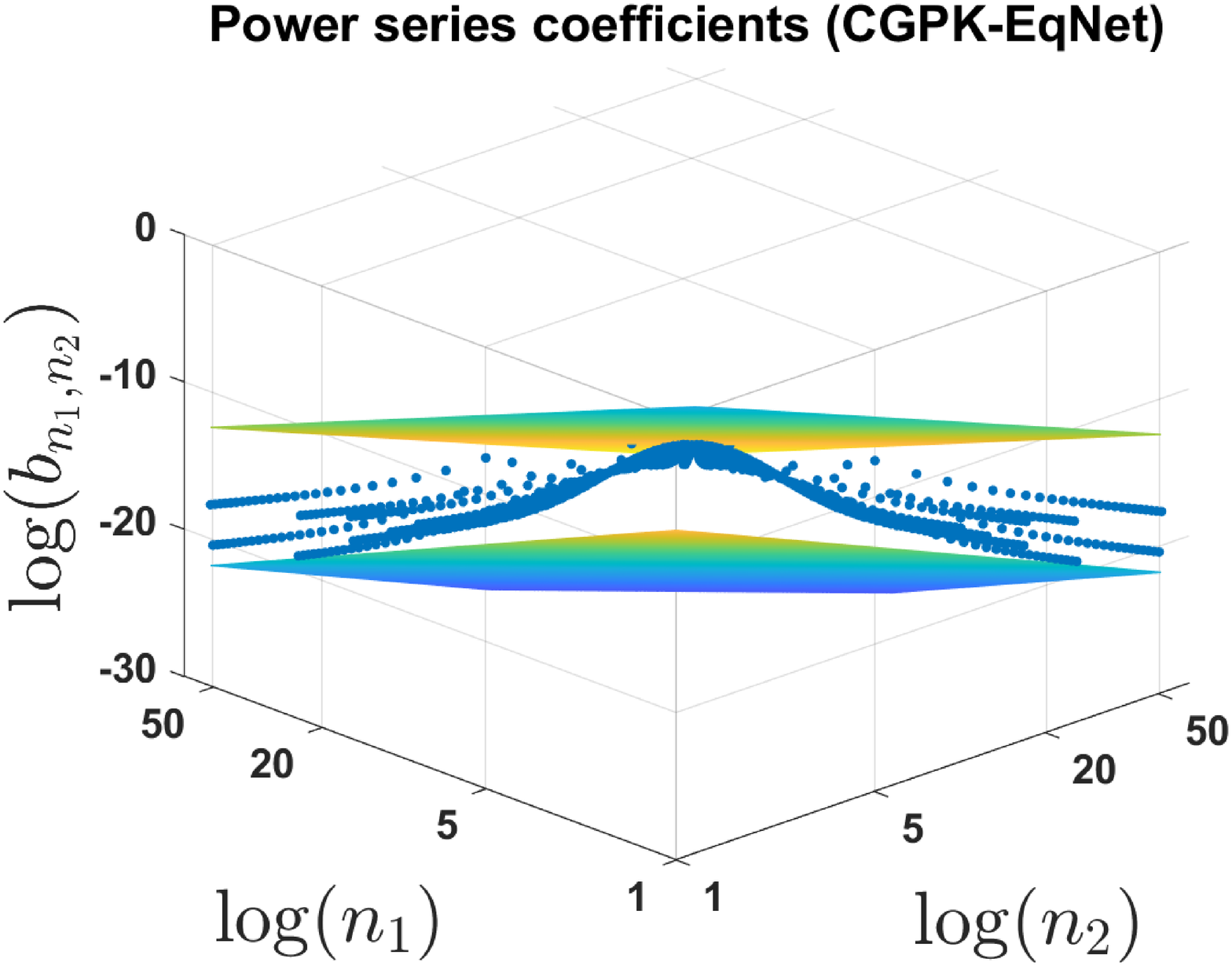}
    \includegraphics[width=0.4\textwidth]{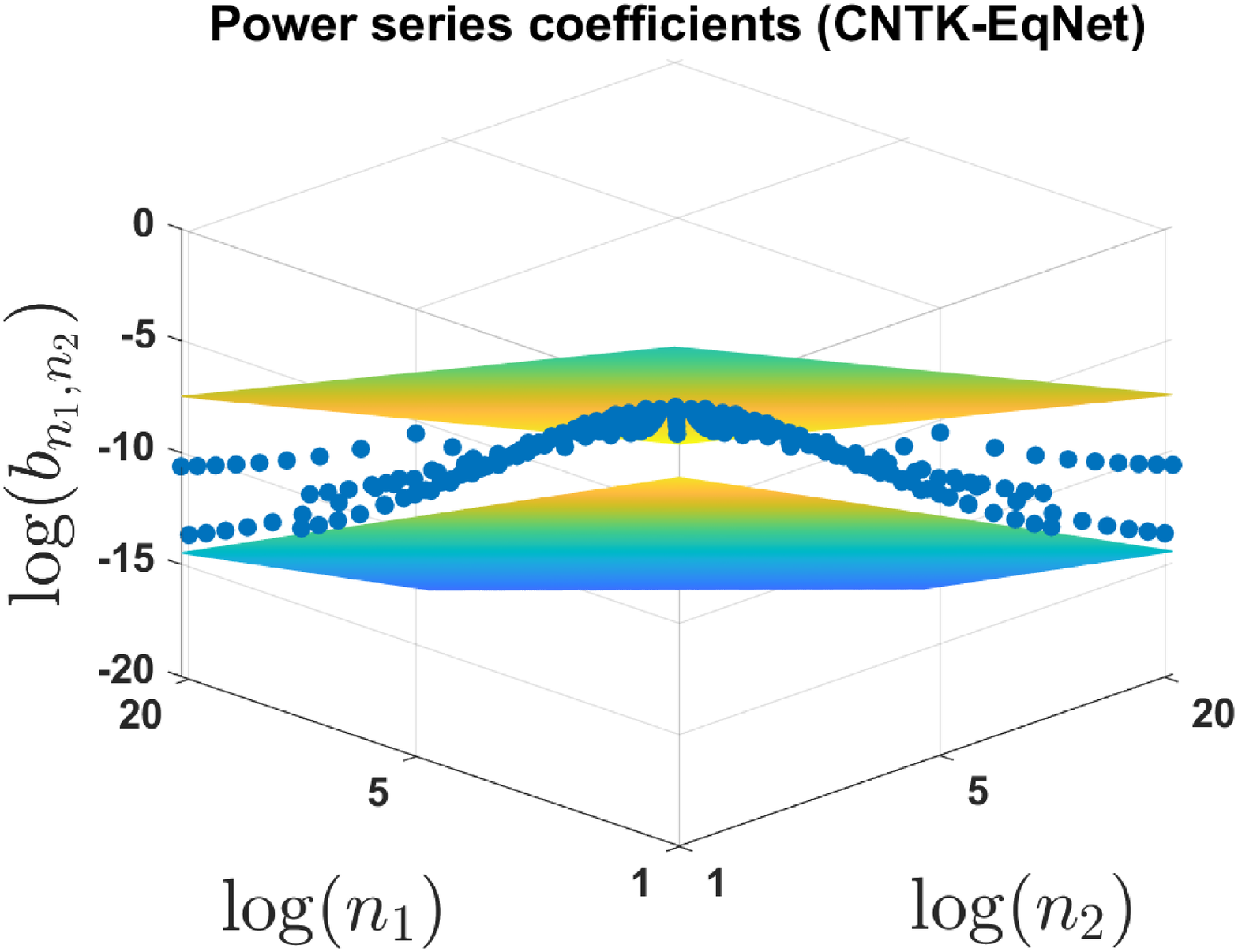}
    \caption{\small The power series coefficients of CGPK-EqNet (top, $n \le 50$) and CNTK-EqNet (bottom, $n \le 20$). Here $d=2$, $q=2$, $L=2$ and the receptive field size is 2. The coefficients, $b_\n$, marked with blue points, are shown as functions of $\n=(n_1,n_2)$. Our bounds are depicted in each graph by two planes whose slopes are determined by the exponents in Theorem~\ref{thm:eqnet-eigs}.}
    \label{fig:coefs-bounds}
    \end{figure}

\begin{comment}

We next apply this theorem to CGPK-EqNet and CNTK-EqNet.
\begin{corollary}
\ag{change the constants}
\begin{enumerate}
    \item The constants $c_1,c_3,\tilde c_1,\tilde c_3$ in Theorems \ref{thm:eqnet-eigs} can be lower bounded by $A^{\bar p_\n^{(L)}}$.
    \item The constants $c_2,c_4,\tilde c_2,\tilde c_4$ in Theorem  \ref{thm:eqnet-eigs} can be upper bounded by $B^{\bar p_\n^{(L)}}$.
\end{enumerate}
\end{corollary}

\end{comment}

\subsection{Trace and GAP kernels} \label{sec:trace}

Our next objective is to characterize the RKHS of CGPK and CNTK and their GAP versions.

\begin{definition}
Let $\kr(\x,\z)$ be a multi-dot product kernel.
We define the respective trace kernel by $\krtrace(\x,\z)=\frac{1}{d}\sum_{i=0}^{d-1} \kr(s_i\x,s_i\z)$ and GAP kernel by $\krgap(\x,\z)=\frac{1}{d^2}\sum_{i=0}^{\dbar-1} \sum_{j=0}^{\dbar-1} \kr(s_i\x,s_j\z)$. 
%We further denote the generator kernel $\kr(\x,\z)$ by $\kreqn(\x,\z)$.
\end{definition}

Clearly, by their definition (Section~\ref{sec:CNTK_formula}), CGPK and CNTK respectively are the trace kernels of CGPK-EqNet and CNTK-EqNet, while CGPK-GAP and CNTK-GAP are their GAP versions.

%\begin{lemma} \label{lemma:separabilty}
%Using the definitions in Section~\ref{sec:CNTK_formula}:
%\begin{enumerate}
%    \item CGPK-EqNet is a multi-dot product kernel. CGPK and CGPK-GAP respectively are trace and GAP kernels generated by CGPK-EqNet.
%    \item CNTK-EqNet, is a multi-dot product kernel. CNTK and CNTK-GAP respectively are trace and GAP kernels generated by CNTK-EqNet.
%\end{enumerate}
%\end{lemma}

The eigenvalues and eigenfunctions of trace and GAP kernels can be derived from their generating kernel, as is established in the following lemma.
\begin{theorem}
\label{thm:GAP}
Let $\kr$ be a multi-dot-product kernel with Mercer's decomposition as in \eqref{eq:mercer}, and let $\krtrace$ and $\krgap$ respectively be its trace and GAP versions. Then,
\begin{enumerate}
    \item $\krtrace(\x,\z)=\sum_{\kk,\jj} \lambtr_{\kk}Y_{\kk,\jj}(\x)Y_{\kk,\jj}(\z)$ with
    \begin{align}
        \lambtr_{\kk} = \frac{1}{d}\sum_{i=0}^{d-1} \lambda_{s_i\kk}, 
    \end{align}
    where $\lambda_{\kk}$ denotes an eigenvalue of $\kr$.
    \item $\krgap(\x,\z)=\sum_{\kk,\jj} \lambtr_{\kk} \tilde Y_{\kk,\jj}(\x) \tilde Y_{\kk,\jj}(\z)$ with 
    \begin{align*}
        \tilde Y_{\kk,\jj}(\x) = \frac{1}{\sqrt{d}} \sum_{i=0}^{{\dbar-1}}Y_{s_i\kk,s_i\jj}(\x).
    \end{align*}
\end{enumerate}
\end{theorem}
 
According to this theorem, the eigenfunctions of a trace kernel are the SH-products. This is simply because the trace kernel itself is multi-dot product. Its eigenvalues are obtained by averaging the eigenvalues of $\kr$ for shifted frequencies $s_i\kk$, making the eigenvalues $\lambda_{\kk}$ invariant to shifts of the index vector $\kk$. Note however that for a trace kernel, the respective eigenfunctions are not shift invariant.  The second part of the lemma establishes that GAP kernels share the same eigenvalues as their respective trace kernels. For these kernels, however, the eigenfunctions consist of sums of SH-products that are shift invariant. 

By combining Theorem.~\ref{thm:GAP} with Theorem~\ref{thm:eqnet-eigs}
%Theorems \ref{thm:cgpk-eqnet-eigs} and \ref{thm:cntk-eqnet-eigs} 
we obtain the following bounds on the eigenvalues of these kernels.

\begin{corollary} \label{cor:eigenvalues_decay}
Let $\kr^{(L)}$ denote either CGPK, CGPK-GAP, CNTK or CNTK-GAP of depth $L$ whose input includes $\qbar$ channels, with receptive field $\R$ and with ReLU activation. Then, the eigenvalues of $\kr^{(L)}$ are bounded by
\begin{align*}
    \sum_{j=0}^{d-1}  & c_3 \prod_{i \in \R, k_{i+j}>0} \tilde A^{\min(p_{i+j}^{(L)},k_{i+j})} k_{i+j}^{-\left(\qbar+2\nu_a-3 \right)} \leq \lambda_\kk \leq \\
    &\sum_{j=0}^{d-1} c_4 \prod_{i \in \R, k_{i+j}>0} k_{i+j}^{-\left(\qbar+2\nu_b-3 \right)},
\end{align*}
where for CGPK, CGPK-GAP $\nu_a=2.5$ and $\nu_b=1+3/(2d)$, while for CNTK, CNTK-GAP $\nu_a=2.5$ and $\nu_b=1+1/(2d)$.
Also, $p^{(L)}_i$ denotes the number of paths from pixel $i$ to the output of $\kr^{(L)}$, $\tilde A>1$ and $c_3$ and $c_4$ depend on $L$.
Here $k_{i+j}$ is identified with $i+j-d$ if $i+j>d$.
\end{corollary}

\begin{comment}

\begin{corollary} \label{cor:cgpk_eigenvalues_decay}
Let $\kr^{(L)}$ denote either CGPK or CGPK-GAP of depth $L$ whose input includes $\qbar$ channels, with receptive field $\R$ and with ReLU activation. Then, the eigenvalues of $\kr^{(L)}$ are bounded by
\begin{align*}
    \sum_{j=0}^{d-1} c_{A,\sgn(s_j\kk)} & \prod_{i \in \R, k_{i+j}>0}  k_{i+j}^{-(\qbar+2)} \leq \lambda_\kk \leq \\
    &\sum_{j=0}^{d-1} c_{B,\sgn(s_j\kk)} \prod_{i \in \R, k_{i+j}>0} k_{i+j}^{-\left(\qbar+\frac{3}{d}-1 \right)},
\end{align*}
where $c_{A,\sgn(\kk)}$, $c_{B,\sgn(\kk)}$, $A=A_{2.5}$, $B=B_{1+3/(2d)}$ are as in Thm.~\ref{thm:hierarchy}.
\end{corollary}

\begin{corollary} \label{cor:cntk_eigenvalues_decay}
Let $\kr^{(L)}$ denote CNTK or CNTK-GAP of depth $L$ whose input includes $\qbar$ channels, with receptive field $\R$ and ReLU activation. Then, the eigenvalues of $\kr^{(L)}$ are bounded by
\begin{align*}
    \sum_{j=0}^{d-1} c_{A,\sgn(s_j\kk)} & \prod_{i \in \R, k_{i+j}>0}  k_{i+j}^{-(\qbar+2)} \leq \lambda_\kk \leq \\
    &\sum_{j=0}^{d-1} c_{B,\sgn(s_j\kk)} \prod_{i \in \R, k_{i+j}>0} k_{i+j}^{-\left(\qbar+\frac{1}{d}-1 \right)},
\end{align*}
where $c_{A,\sgn(\kk)}$, $c_{B,\sgn(\kk)}$, $A=A_{2.5}$, $B=B_{1+1/(2d)}$ are as in Thm.~\ref{thm:hierarchy}.
\end{corollary}

\end{comment}

Overall, Corollary \ref{cor:eigenvalues_decay} indicates
%\ref{cor:cgpk_eigenvalues_decay} and \ref{cor:cntk_eigenvalues_decay} indicate 
that the eigenvalues of the trace and GAP CGPK and CNTK decay polynomially with bounds that are similar to those of the equivariant kernel. Here too, the exponent depends on the number of channels and forms a bias toward learning target functions that depend on few pixels. The multiplicative coefficients depend on the number of paths from each pixel to the output. In contrast to the equivariant kernels, the eigenvalues of the trace and GAP kernels are averages over shifted indices, and therefore they depend on the \textit{relative} positions of pixels. For example, the eigenvalues of SH-products involving only two pixels with large frequencies, $k_i$ and $k_{i+\delta}$, is large when $\delta$ is small, and is smaller when $\delta$ is large. This is illustrated in the graph in Figure \ref{fig:coefs}, which shows the relative magnitude of the eigenvalues, according to the lower bound, involving exactly two pixels as a function of their distance. It can be seen that the coefficients decay exponentially with the distance between the two pixels. Moreover, for a fixed receptive field size, a more rapid decay is obtained with a smaller filter and deeper architecture, compared to a larger filter and shallower network.

\begin{figure}[tb]
\centering
\includegraphics[width=0.33\textwidth]{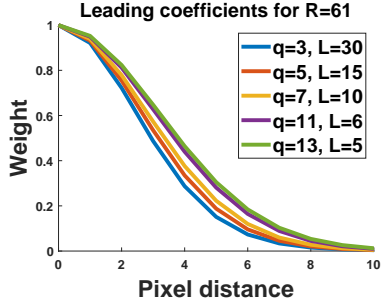}
\caption{\small The figure shows the relative magnitude of the eigenvalues of a hierarchical trace kernel for SH-products involving two pixels as a function of the distance between the pixels. This is shown for various architectures with different depths $L$ and convolution filter sizes $q$ that results in the same receptive field size of 61.}
\label{fig:coefs}
\end{figure}

The graphs in Figure \ref{fig:numeric} show the eigenvalues of the CGPK kernel evaluated numerically for a three-layer network. The figure shows the eigenvalues for frequencies $k$ in 1-4 pixels. The slope of these lines, which depicts the exponent of the decay of the corresponding eigenvalues, is in good correspondence with our theory. Moreover, the decay rate is higher for frequencies spread over more pixels. This can be contrasted with the eigenvalues for the Gaussian Process Kernel for a fully connected network (FC-GPK) shown in Figure~\ref{fig:numeric_decay_ntk}. For that kernel the decay rate is the same regardless of the pixel spread of frequencies, as is indicated by the parallel lines in the graph, and is empirically roughly equal to the maximal decay rate obtained for CGPK with 4 pixels. Figure~\ref{fig:numeric} further shows the eigenvalues obtained with two frequency patterns, $(k_1,k_2,0,0,...)$ and $(k_1,0,k_2,0,...)$. As is predicted by our theory, while the decay rate for both patterns is similar, the latter eigenvalues are smaller than the former ones by a multiplicative constant.

\begin{figure}[tb]
    \centering
    \includegraphics[width=0.4\textwidth]{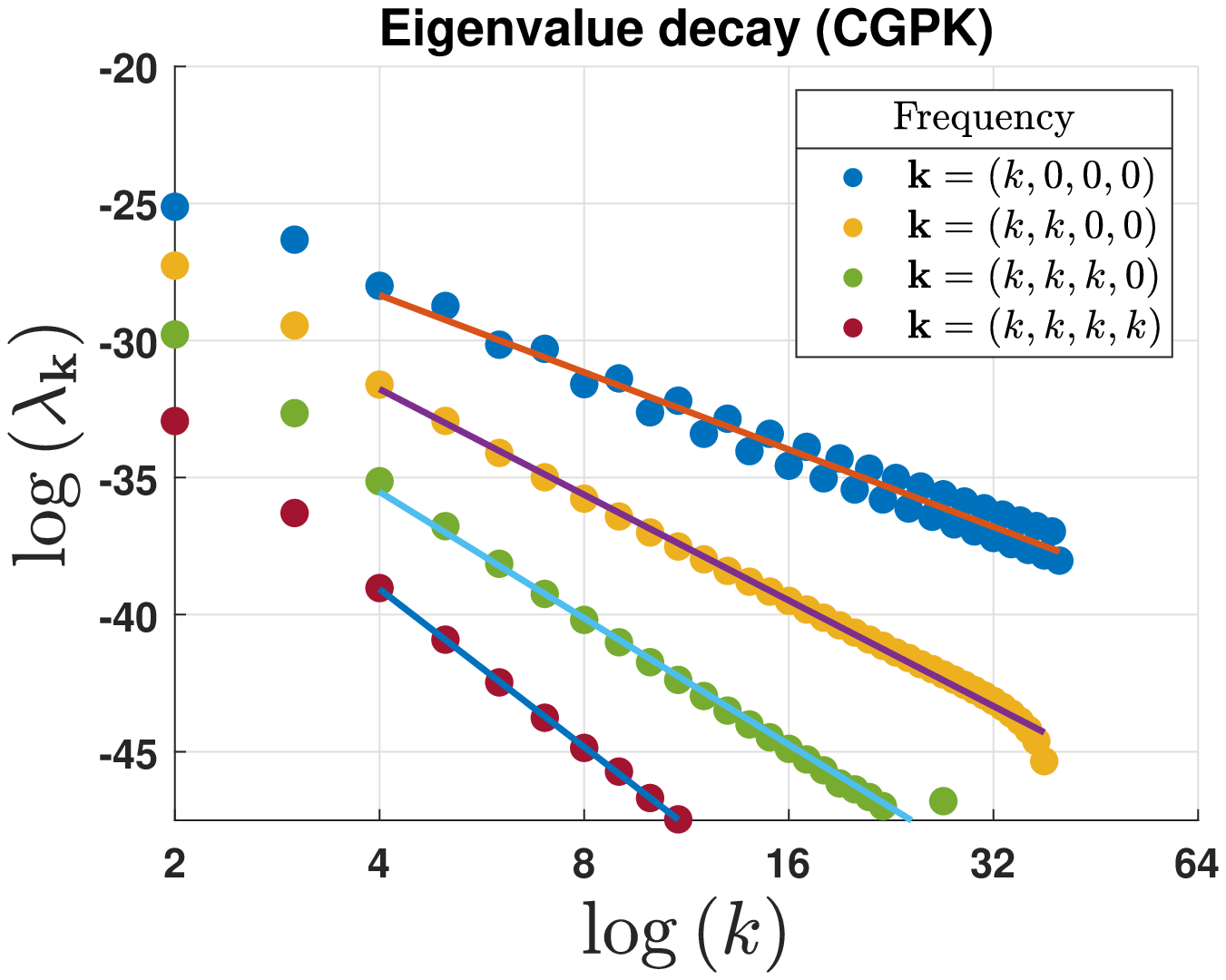}
    \includegraphics[width=0.4\textwidth]{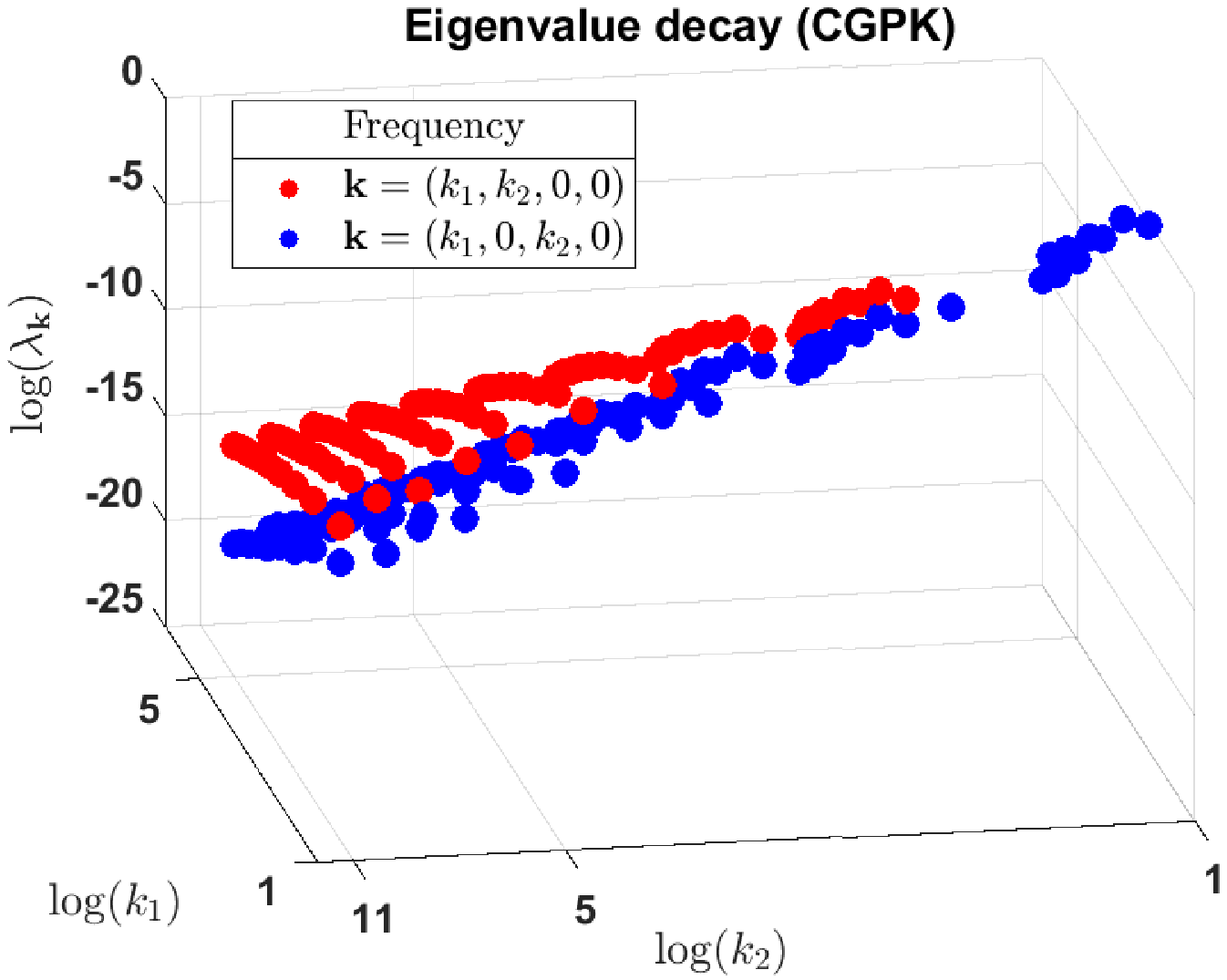}
    \caption{\small Left: The eigenvalues of CGPK for frequency patterns that include either one non zero frequency (blue dots), two (orange), three (green) or four (maroon) identical frequencies. The slopes (respectively, $-5.63$, $-7.71$, $-9.3$ and $-11.5$) indicate the exponent for each pattern. ($d=4$, $\qbar=3$, $q=2$, $L=3$).
    Right: The eigenvalues of CGPK for two frequency patterns that include exactly two non zero frequency, either next to each other (red dots) or separated by one zero frequency (blue). Consistent with our theory, the eigenvalues obtained with non-zero frequencies next to each other are larger than those obtained with the same frequencies but separated ($d=4$, $\qbar=3$, $q=2$, $L=3$).}
    \label{fig:numeric}
\end{figure}

\begin{figure}[tb]
    \centering
    \includegraphics[width=0.4\textwidth]{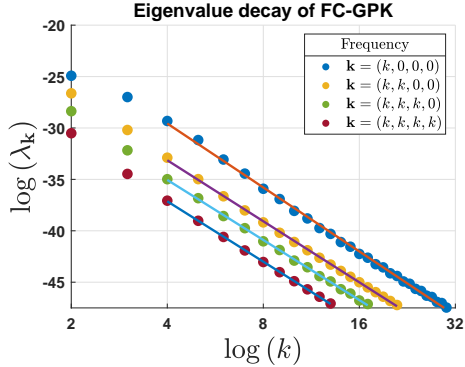}
    \caption{\small The eigenvalues of FC-GPK on $\ms$ for frequency patterns that include either one non zero frequency (blue dots), two (yellow), three (orange) or four (maroon) identical frequencies. The slopes (respectively, $-12.4$, $-11.8$, $-11.6$ and $-11.7$) indicate the exponent for each pattern. ($d=4$, $\qbar=2$, $q=2$, $L=3$ ).}
    \label{fig:numeric_decay_ntk}
\end{figure}

\begin{figure}[t]
    \centering
    \includegraphics[width=0.4\textwidth]{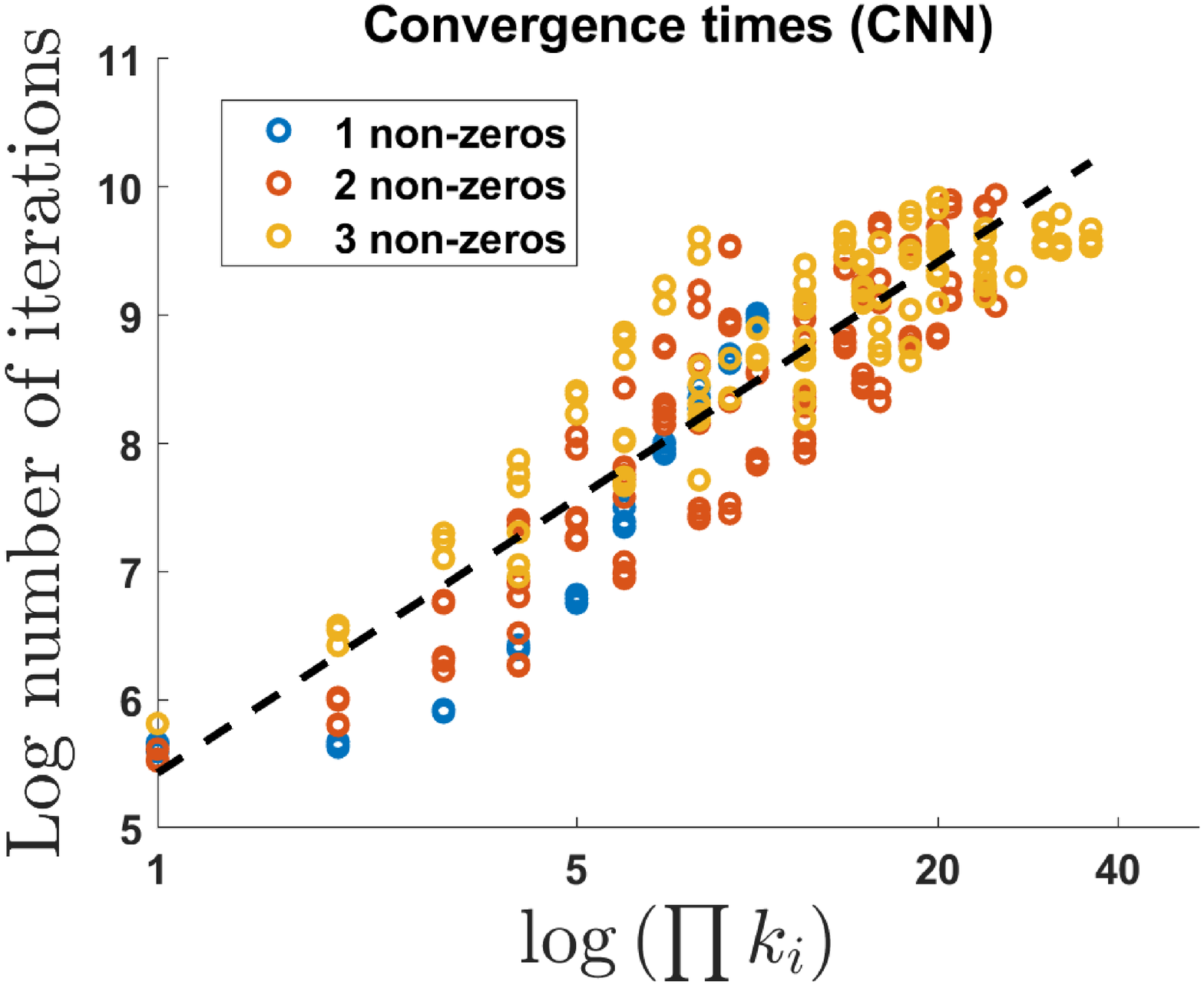}
    \includegraphics[width=0.4\textwidth]{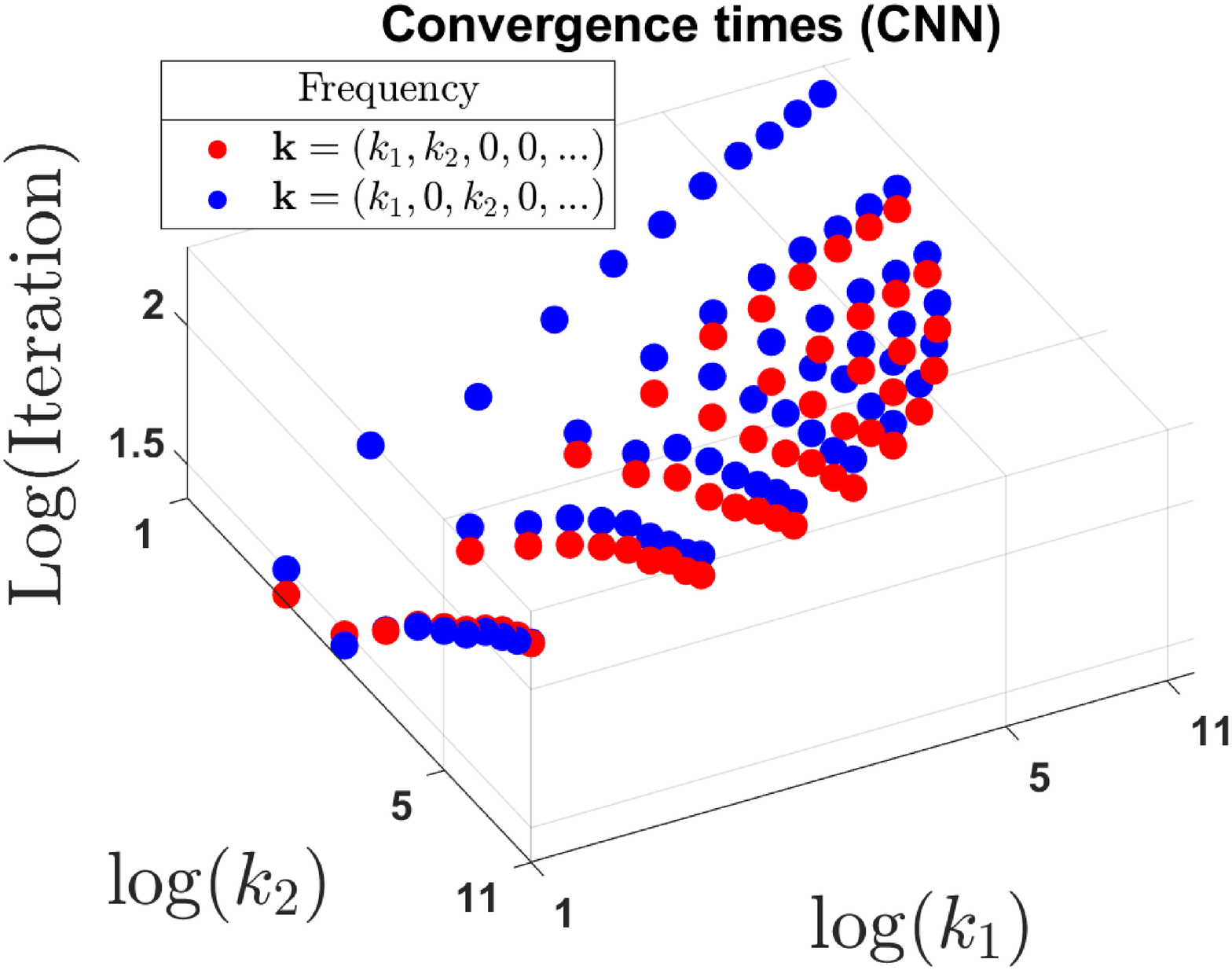}    
    \caption{\small Training a CNN to fit eigenfunctions of the CNTK kernel ($\qbar=2$, $d=8$, $q=3$, $L=3$, $m=1000$). Each experiment is indicated by a circle or dot, depicting the number of GD iterations required to reach a $10^{-4}$ error as a function of (the product of) frequencies. Recall that, to the extent that GD for the CNN is similar to kernel GD of CNTK, the number of iterations needed to learn an eigenfunction should be inversely proportional to the corresponding eigenvalue. Left: The dashed line depicts a linear regression of these experiments. Its slope should reflect (up to a sign) the average decay of the eigenvalues.  With a slope of 1.3, this regression line falls between the bounds (1.125 and 4) in Corollary \ref{cor:eigenvalues_decay}. 
    Right: Similar to Figure~\ref{fig:numeric}, we tested two frequency patterns. Convergence times for adjacent non-zero frequencies (red dots) are faster than their corresponding frequencies that are separated by one pixel (blue dots).}
    \label{fig:network}
\end{figure}

\begin{figure}[!]
    \centering
    \includegraphics[width=0.4\textwidth]{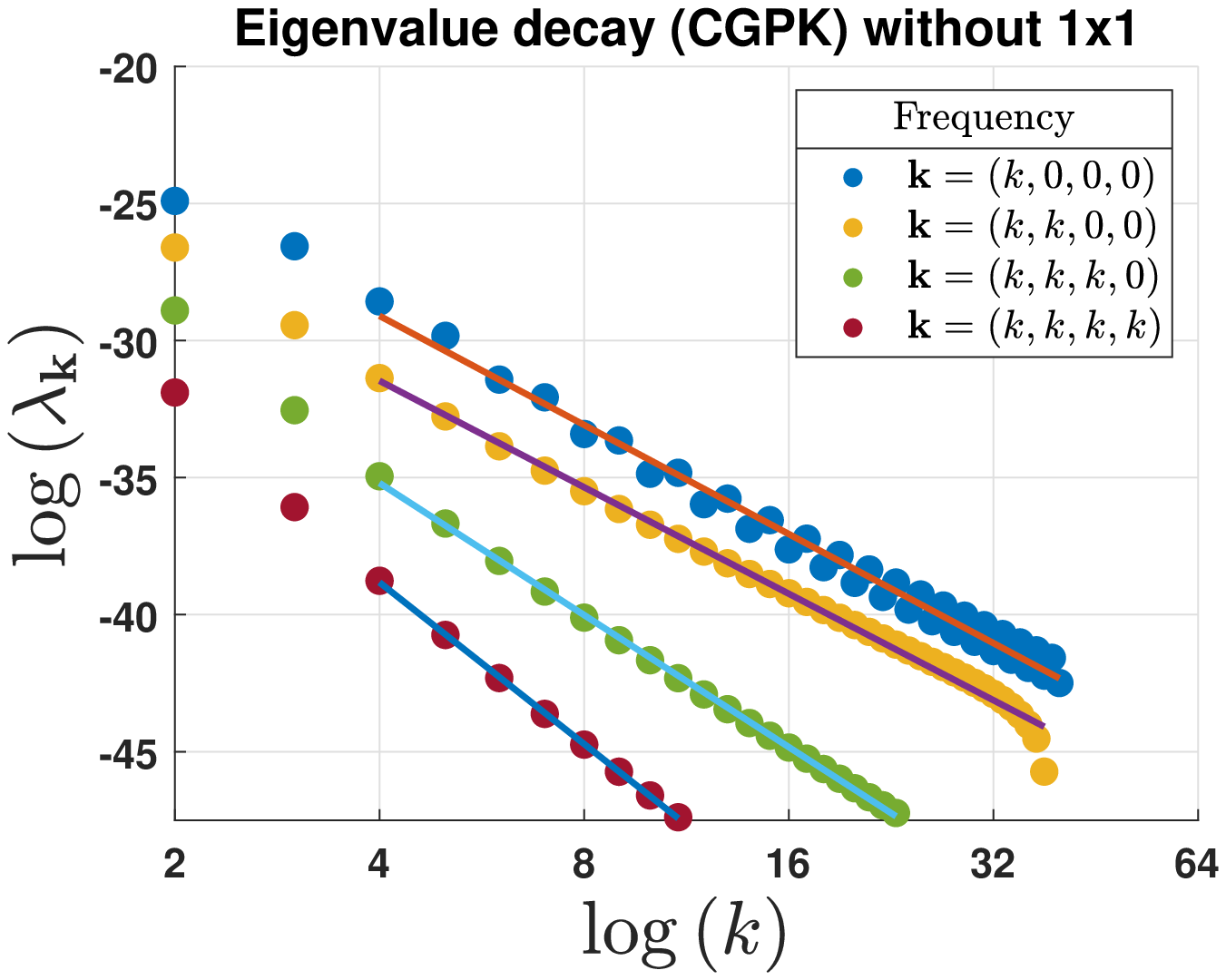}
    \includegraphics[width=0.4\textwidth]{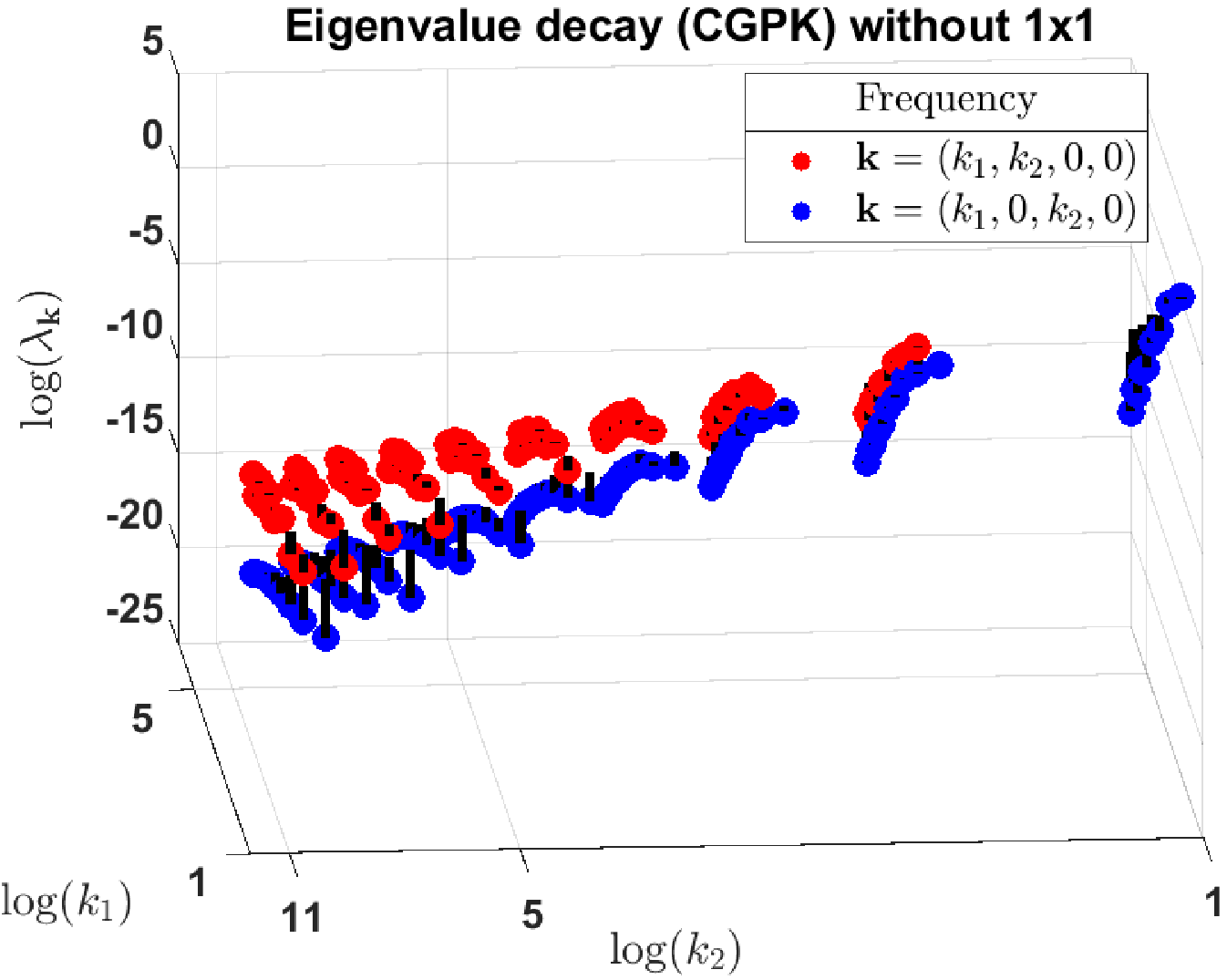}
    \caption{\small The eigenvalues of CGPK without the $1 \times 1$ convolution. Left: Eigenvalues for different frequency patterns that include either one non zero frequency (blue dots), two (red), three (orange) or four (maroon) identical frequencies. The slopes (respectively, $-7.3$, $-7.5$, $-9.2$ and $-11.7$) indicate the exponent for each pattern. ($d=4$, $\qbar=2$, $q=2$, $L=3$ ). Right: Eigenvalues for two frequency patterns that include exactly two non zero frequency, either next to each other (red dots) or separated by one zero frequency (blue). Black lines indicate differences from the values shown in Figure~\ref{fig:numeric}.}
    \label{fig:numeric_decay_no_1x1}
\end{figure}

To examine the relevance of this analysis to CNNs of moderate size we trained CNNs to regress products of trigonometric functions, which are the eigenfunctions of the CNTK trace kernel with $\qbar=2$. Figure~\ref{fig:network} shows for each experiment the number of GD iterations needed to achieve a prescribed error. Recall that, to the extent that GD for the CNN is similar to kernel GD of CNTK, the number of iterations needed to learn an eigenfunction should be inversely proportional to the corresponding eigenvalue. The figure shows that overall the runtime lies between our predicted bounds. It also shows that it is faster to learn eigenfunctions involving two non zero frequencies adjacent to each other than ones that involve non-zero frequencies separated by one pixel, as is predicted by our theory. 

Finally, our analysis uses network models which include a $1 \times 1$ convolution layer as the first layer. This simplifies our derivations and appears to have only limited effect on the results. In particular, it can be readily shown that for data distributed uniformly on the multisphere, the eigenfunctions of CGPK and CNTK without the $1 \times 1$ convolution are the SH-products (or their shift invariant sums in the case of GAP kernels). Figure~\ref{fig:numeric_decay_no_1x1} shows the eigenvalues of CGPK under the same conditions as in Figure~\ref{fig:numeric} with the $1 \times 1$ convolution removed. Overall, similar decay patterns are observed, except that in the case of high frequency in one pixel (denoted '$(k,0,0,0)$') the eigenvalues decay faster than with the $1 \times 1$ convolution. 
%We conjecture that this holds for larger networks as well, i.e, the decay of the eigenvalues when the high frequencies are concentrated in any subset of the filter size should be roughly the same. 

\section{Related Work}

A number of papers analyze the spectral properties and generalization bounds for convolutional kernels, including CNTK. \cite{bietti2019inductive} derived explicit feature maps for CNTK. \cite{misiakiewicz2021learning} developed spectral decomposition and generalization bounds for CNTK with one convolution layer and pooling with input in the nodes of the hyper-cube $\{-1,1\}^d$. \cite{favero2021locality} considered CNTK with one convolutional layer and non-overlapping patches. Such input is equivalent to just $1 \times 1$ convolution layer in our setup. They developed a spectral theory for this case and used it to derive new generalization bounds. They showed that the eigenvalues decay is dictated by the patch size (corresponding to the number of channels in our setup). Their limited setup yields Mercer's decomposition similar to that obtained with NTK for a two-layer fully connected network. 
%Note that one layer convolution CNTK with filter smaller than the signal size is not universal. 

\cite{mei2021learning} considered a network with one convolution layer with patches of full image size and inputs drawn from the uniform distribution on either the sphere or the hypercube. Their work focuses on the benefits of group invariance in reducing the sample size. \cite{bietti2021sample} considered general kernels that incorporate group invariance and derived a generalization bound based on counting the number of eigenfunctions of the kernel. 

\cite{mairal2014convolutional,mairal2016end} proposed a convolutional kernel network (CKN) that involves layers of patch extraction, convolution, and pooling. Pooling in this model is implemented by a convolution with a fixed filter such as a Gaussian, possibly followed by subsampling. We note that without subsampling, such pooling operation can effectively result in a fully connected layer. \cite{scetbon2020harmonic} proved that the set of functions produced by the CKN network that applies convolution with non-overlapping patches is contained in the RKHS of a kernel that includes just one convolution layer. Similar to our work, their work used inputs drawn from a multisphere, yielding eigenfunctions that are SH-products. They however bounded the eigenvalue decay for CKNs with polynomial activations at all except the first layer (with no convolutions in those layers). \cite{bietti2021approximation} extended this setup to enable convolution with overlapping patches and further derive generalization bounds. In contrast to these works, our work shows that the SH-products are the eigenfunctions of all multi-dot product kernels and further provides spectral analysis for CGPK and CNTK for multilayer CNNs.

\cite{heckel2019denoising} investigated the speed of convergence of gradient descent for equivariant networks with a fully connected layer followed by a fixed convolution layer in the context of image denoising. 
In a related work \cite{Tachella_2021_CVPR} used NTK for a two-layer network to relate denoising with CNNs to denoising with non-local filters.

Existing work also investigates CNNs from a non-kernel perspective, for example, by showing that multi-layer CNNs (and related networks) can efficiently learn \textit{compositional} functions \cite{poggio2017and,cohen2016expressive}.

%Our work is the first to derive spectral analysis of CNTK. Our formulation sheds light on the inductive bias of CNTK and specifically enables to identify which function will be learnt first when running SGD on a CNN. Moreover, our spectral analysis may be used to specify which functions are contained in the RKHS of CNTK in clear and concise manner. Finally, we analyze the hierarchy corresponds to the CNTK and show how it affects the spectral properties of the kernel.

\section{Conclusion}

Our paper has derived the eigenfunctions and corresponding eigenvalues of neural tangent kernels that describe overparameterized CNNs.  This provides us with a clear and specific understanding of these networks, just as much prior work has done for fully connected networks.  Of particular interest, our work provides a concrete understanding of the inductive bias produced by the hierarchical structure of deep CNNs.  We see that CNNs can efficiently learn higher frequency functions than FC networks when these functions are spatially localized.  We provide formulas that can show the trade-off points between higher frequency functions and spatial localization, and also show how architectural variations can affect these biases.  We feel that this is a significant step towards understanding which features CNNs will learn, and how this may depend on network architecture.

\subsection*{Acknowledgement}
This research is partially supported by the Israeli Council for Higher Education (CHE) via the Weizmann Data Science Research Center and by research grants from the Estate of Tully and Michele Plesser and the Anita James Rosen Foundation. David Jacobs is supported by the Guaranteeing AI Robustness Against Deception(GARD) program from DARPA and the National Science Foundation under grant no.~IIS-1910132.

\bibliographystyle{style.sty}
\bibliography{main.bib}

\newpage
\onecolumn
\appendix

\section*{Appendix}

\begin{comment}

\section{Extension to other deep learning model}

\ag{New section}
In this section we discuss variants of the neural network model presented in \ref{tab:network_model} and their consequences on the results.

\textbf{CNN without one by one convolution}- changing the model in \ref{fig:network_model} to include arbitrary convolution filter size in the first layer, is resulted in the CNN as people use in practice. The effect of such change has on the CNTK/CGPK in formula appears in Table \ref{tab:cgpk_cntk} is replacing the first layer as:
\begin{align*}
    [\Sigma^{(0)}(\x,\z)]_{i,j} = [\Theta^{(0)}(\x,\z)]_{i,j} = \sum_{r=0}^{q-1}[X^TZ]_{i+r,j+r}
\end{align*}
Where $[A]_{i,j}$ denotes the $i,j$ entry in the matrix $A$. The formula above implies that CNN without one by one convolution is a multi-dot product kernel as well and therefore the results of the eigenfunction applies to this case. A plot of the eigenvalues appears in figure \ref{fig:numeric_decay_no_1x1}. It can be seen that the decay is fairly similar to the one in figure \ref{fig:numeric} except that the decay of a single pixel is a bit higher.

\end{comment}

\section{Multi-dot product kernels} \label{app:multidot}

In this section we prove results presented in Section~\ref{sec:multidot}.

\begin{lemma}\label{lemma:multi_dot_product}
CGPK-EqNet and CNTK-EqNet are multi-dot product kernels.
\end{lemma}

\begin{proof}
The proof follows directly from the derivation of the kernels in Section~\ref{sec:CNTK_formula}. Note that CGPK-EqNet is given by $\tilde\Sigma_{11}^{(L)}$ and CNTK-EqNet by $\tilde\Theta_{11}^{(L)}$, and their recursive definition only involve elements of $\tilde\Sigma^{(l)}$ and $\tilde\Theta^{(l)}$ in which $i=j$. Moreover, by definition the diagonal elements of $\Sigma^{(0)}$ and $\Theta^{(0)}$ are $\langle\x^{(i)},\z^{(i)}\rangle$ for $i \in [d]$, implying the lemma.
\end{proof}

\subsection{Multivariate Gegenbauer Polynomials}

We next extend basic results derived for functions on the sphere to the multisphere $\ms$. These results will assist us later to prove Mercer's decomposition for multi-dot product kernels in the subsequent section.

We consider the set of  Gegenbauer polynomials $\{Q_k^{(\qbar)}(t)\}_{k \ge 0}$ that are orthogonal in $L_2[-1,1]$ w.r.t.\ the weight function $(1-t^2)^{(\qbar -3)/2}$ and omit the superscript. 
Inspired by \cite{dhahri2014multi}, %which defined general multivariate orthogonal polynomials, 
we define  multivariate Gegenbauer polynomials, using facts from harmonic analysis on the sphere. (See references \cite{groemer1996geometric,muller2012analysis} for background on spherical harmonics and Gegenbauer polynomials). We denote by $|\Sphere^{\qbar-1}|$ the area of the sphere $\Sphere^{\qbar-1}$.
\begin{definition}
For $k\geq 0$, let $Q_k(t):[-1,1]\rightarrow \Real $ be the (univariate) Gegenbauer polynomial of degree $k$. Then, the multivariate Gegenbauer polynomial of order $\kk$ is $Q_{\kk}(\tbf):[-1,1]^\dbar\rightarrow \Real$, defined by
\begin{align*}
   Q_{\kk}(\tbf)=Q_{k_1}(t_1)\cdot Q_{k_2}(t_2)\cdot ...\cdot Q_{k_{\dbar}}(t_{\dbar}).
\end{align*}
\end{definition}
These multivariate Gegenbauer polynomials enjoy several properties that they inherit from their univariate counterpart.

\begin{lemma}\label{lemma:orth_1d}
Let $P_k(\tbf)$ denote the space of polynomials of degree $\leq k$ with variables $\tbf \in [-1,1]^\dbar$. Then, the set $\{Q_{\ii}(\tbf) \}_{\ii=0}^{|\ii|=k}$ is an orthogonal basis of $P_k(\tbf)$ w.r.t.\ the weight function $(1-t^2)^{(\qbar -3)/2}$ (with $\ii=(i_1,\ldots,i_{\dbar})$ and $|\ii| = i_1+\ldots+i_{\dbar}$).
\end{lemma}

\begin{proof}
Let $p(\tbf)=\sum_{\ii=0}^{|\ii|=k}a_{\ii}\tbf^\ii\in P_k(\tbf)$.
Since the univariate Gegenbauer polynomials form an orthogonal basis, for every $0\leq n_i\leq k$ and $i\in [\dbar]$ we can write
$t_i^{n_i}=\sum_{j=0}^{n_i} a^{(n_i)}_jQ_j(t_i)$,
where $a_j^{(n_i)}\in \Real$, and the superscript is used to emphasize that the expansion depends on $n_i$. Therefore, $p(\tbf)$ can be written as
\begin{align*}
    p(\tbf)&=\sum_{\n=0}^{|\n|=k}a_{\n}\tbf^\n=\sum_{\n=0}^{|\n|=k}a_{\n}\left(\sum_{j=0}^{n_1} a^{(n_1)}_jQ_j(t_1)\right)\cdot ..\cdot \left(\sum_{j=0}^{n_d} a^{(n_d)}_jQ_j(t_{\dbar})\right)\\
    &=^{(1)}\sum_{\n=0}^{|\n|=k}\tilde{a}_{\n}Q_{n_1}(t_1)Q_{i_2}(t_2)\cdot \ldots \cdot Q_{n_d}(t_d)=\sum_{\n=0}^{|\n|=k}\tilde{a}_{\n}Q_{\n}(\tbf),
\end{align*}
where $^{(1)}$ is obtained by applying the distributive law with the fact that $n_1,..,n_{\dbar}\leq k$. finally,  $\tilde{a}_{\ii}$ can be computed explicitly from $a_{\ii}$ and $\{a_j^{(n_i)}\}$.

We have shown that $P_k(\tbf)$ is spanned by the set $\{Q_{\ii}(\tbf) \}_{\ii=0}^{|\ii|=k}$. Next, we show that this set is orthogonal with respect to the measure $\prod_{r=1}^d \left((1-t_r^2)^{\frac{\qbar-3}{2}}\right)$. Let $\ii$ and $\jj$ be two vectors of indices. Then, we have that
\begin{align*}
    \int_{[-1,1]^d}&Q_{\ii}(\tbf)Q_{\jj}(\tbf)\prod_{r=1}^d (1-t_r^2)^{\frac{\qbar-3}{2}}dt_1 \cdot \ldots \cdot dt_{\dbar}
    = \prod_{r=1}^d \left(\int_{[-1,1]} Q_{i_r}(t_r)Q_{j_r}(t_r) (1-t_r^2)^{\frac{\qbar-3}{2}}dt_r\right)\\
    %=&\left(\int_{[-1,1]}Q_{i_1}(t_1)Q_{j_1}(t_1) (1-t_1^2)^{\frac{\qbar-3}{2}}dt_1\right)\cdot ...\cdot \left(\int_{[-1,1]}Q_{i_{\dbar}}(t_{\dbar})Q_{j_{\dbar}}(t_{\dbar}) (1-t_{\dbar}^2)^{\frac{\qbar-3}{2}}dt_{\dbar}\right)\\
    =&\left(\frac{|\Sphere^{\qbar-1}|}{|\Sphere^{\qbar-2}|}\right)^d \left(\prod _{r=1}^{d}N(\qbar,i_r)\right)^{-1}\delta_{i_1,j_1}\cdot \delta_{i_2,j_2}\cdot \ldots \cdot \delta_{i_{\dbar},j_{\dbar}},
\end{align*}
where the last equality is due to the orthogonality property of the univariate Gegenbauer polynomials. 
This concludes the proof.
\end{proof}

The relation of the multivariate Gegenbauer polynomials to the SH-products is formulated in the following lemma.
\begin{lemma}\label{lemma:addition_thm_multi}
Let  $\x,\z\in \ms$. It holds that
    \begin{align*}
    Q_{\kk}(\langle \x^{(1)} ,\z^{(1)} \rangle ,..,\langle \x^{(i)} ,\z^{(j)} \rangle,..,\langle \x^{({\dbar})} ,\z^{({\dbar})} \rangle) = |\Sphere^{\qbar-1}|^d \left( \prod _{r=1}^{d} N(q,k_r)\right)^{-1}\sum_{\jj:j_r\in [N(\qbar,k_r)]}Y_{\kk,\jj}(\x)Y_{\kk,\jj}(\z),
    \end{align*}
    where $Y_{\kk,\jj}(\x)$ is homogeneous polynomial of degree $k_1+..+k_{\dbar}$. $Y_{\kk,\jj}(\x)$ is further given by SH-products, i.e., $Y_{\kk,\jj}(\x)=   \prod_{i=1}^{{\dbar}} Y_{k_i,j_i}(\x^{(i)})$,   
    where $Y_{k_ij_i}$ are spherical harmonics in $\Sphere^{\qbar-1}$, and $N(\qbar,k_i)$ are the number of harmonics of frequency $k_i$ in $\Sphere^{\qbar-1}$.
\end{lemma}

\begin{proof}
By the definition of the multivariate Gegenbauer polynomials and the univariate addition theorem \cite{smola2001regularization} we get
\begin{align*}
    Q_{\kk}&(\langle \x^{(1)} ,\z^{(1)} \rangle ,...,\langle \x^{(i)} ,\z^{(j)} \rangle,...,\langle \x^{({\dbar})} ,\z^{({\dbar})} \rangle)=Q_{k_1}(\langle \x^{(1)} ,\z^{(1)} \rangle )\cdot ...\cdot Q_{k_{\dbar}}(\langle \x^{({\dbar})} ,\z^{({\dbar})} \rangle)\\
    =&\left(\frac{|\Sphere^{\qbar-1}|}{N(\qbar,k_1)}\sum_{j_1=1}^{N(\qbar,k_1)}Y_{k_1,j_1}(\x^{(1)})Y_{k_1,j_1}(\z^{(1)})\right)\cdot \ldots \cdot \left(\frac{|\Sphere^{\qbar-1}|}{N(\qbar,k_{\dbar})}\sum_{j_{\dbar}=1}^{N(\qbar,k_{\dbar})}Y_{k_{\dbar},j_{\dbar}}(\x^{({\dbar})})Y_{k_{\dbar},j_{\dbar}}(\z^{({\dbar})})\right)\\
    =&\left(\prod_{i=1}^{{\dbar}}\frac{|\Sphere^{\qbar-1}|}{N(\qbar,k_i)}\right)\sum_{\jj=(1,...,1)}^{\jj=(N(\qbar,k_1),...,N(\qbar,k_{\dbar}))} \prod_{i=1}^{{\dbar}}Y_{k_i,j_i}(\x^{(i)})Y_{k_i,j_i}(\z^{(i)})\\
    %&=\left(\prod_{r=1}^{{\dbar}}\frac{|S^{\qbar-1}|}{N(\qbar,k_r)}\right)\sum_{\jj=(1,..,1)}^{\jj=(N(\qbar,k_1),..,N(\qbar,k_{\dbar}))} \left(\prod_{r=1}^{{\dbar}}Y_{k_r,j_r}(\x^{(r)})\right)\left(\prod_{r=1}^{{\dbar}}Y_{k_r,j_r}(\z^{(r)})\right)\\
    :=&\left(\prod_{i=1}^{{\dbar}}\frac{|\Sphere^{\qbar-1}|}{N(\qbar,k_i)}\right)\sum_{\jj:j_i\in [N(\qbar,k_i)]} Y_{\kk,\jj}(\x)Y_{\kk,\jj}(\z).
\end{align*}
%\mg{There is SH literature where the area of the sphere is not part of the scaling in the addition formula.}  
Note that the homogeneity of the SH-products $Y_{\kk,\jj}(\x)$  is a direct result of the homogeneity of the spherical harmonics  $Y_{k_i,j_i}$. 
\end{proof}

\begin{lemma}\label{lemma:orth_multi}
The set $\{Y_{\kk,\jj}\}$ are orthonormal w.r.t uniform measure in $\ms$.
\end{lemma}
\begin{proof}
We have that 
\begin{align*}
    \int_{\ms} Y_{\kk,\jj}(\x) Y_{\kk',\jj'}(\x)d\x   =& \int_{\ms}\left(\prod_{i=1}^{{\dbar}}Y_{k_i,j_i}(\x^{(i)})\right)\left(\prod_{i=1}^{{\dbar}}Y_{k'_i,j'_i}(\x^{(i)})\right)d\x \\
    =&\prod_{i=1}^{{\dbar}} \left(\int_{\Sphere^{\qbar-1}}Y_{k_i,j_i}(\x^{(i)})Y_{k'_i,j'_i}(\x^{(i)})d\x^{(i)}\right)=\prod_{i=1}^{d} \delta_{k_i,k'_i}\cdot \delta_{j_i,j'_i}.
\end{align*}
\end{proof}

\subsection{Mercer's decomposition}

In this section we prove that the eigenfunctions of multi dot-product kernels consist of products of spherical harmonics. We further provide a way to calculate the eigenvalues using products of Gegenbauer polynomials.

\begin{lemma} \label{applemma:mercer}
Let $\kr$ be a multi-dot product kernel. Then, the eigenfunctions of $\kr(\x,\cdot)$ w.r.t uniform measure on $\ms$ are the SH-products. Namely, the eigenfunctions are  
\begin{align*}
    \left\{ Y_{\kk,\jj}(\x) = \prod_{i=1}^{{\dbar}} Y_{k_ij_i}\left(\x^{(i)}\right) \right\}_{\kk \geq 0,~ j_i\in [N(q,k_i)]},
\end{align*} 
where $Y_{k_ij_i}$ are the Spherical Harmonics in $\Sphere^{\qbar-1}$, and $N(\qbar,k_i)$ are the number of harmonics of frequency $k_i$ in $\Sphere^{\qbar-1}$.  The eigenvalues, $\lambda_{\kk}$, can be calculated using products of (univariate) Gegenbauer polynomials as follows,
    \begin{align*}
        \lambda_{\kk}= C(\qbar,d)\int_{[-1,1]^d} \kr(\tbf) \prod_{i=1}^d Q_{k_i}(t_i)(1-t_i^2)^{\frac{\qbar-3}{2}}d\tbf
    \end{align*}
    where $\{Q_k(t)\}$ is the set of orthogonal Gegenbauer polynomials w.r.t the weights $(1-t_i^2)^{\frac{\qbar-3}{2}}$,  and $C(\qbar,d)$ is a constant that depends on both $\qbar$ and $d$.
\end{lemma}
\begin{proof}
Let $\kr$ be a multi-dot product kernel. By definition for such kernel, there exists a multivariate analytic  function $\kappa$ such that $\kr^{(L)}(\x,\z) = \kappa(\langle \x^{(1)},\z^{(1)}\rangle, ..., \langle \x^{(d)},\z^{(d)}\rangle)$. Using lemma \ref{lemma:orth_1d}, $\{Q_{\kk}\}$ form an orthogonal basis in $[-1,1]^d$. Therefore, it can be readily shown (similar to \cite{smola2001regularization}) that,  $\kappa$ can be written as 
\begin{align*}
    \kappa(t_1,..,t_{\dbar}):=\kappa(\tbf) = \sum_{\kk \geq 0} \left(\prod_{i=1}^{d}N(\qbar,k_i) \frac{|\Sphere^{\qbar-2}|}{|\Sphere^{\qbar-1}|} \right) Q_{\kk}(\tbf)\int_{[-1,1]^\dbar}\kappa(\tilde{\tbf})Q_{\kk}(\tilde{\tbf})\prod_{i=1}^d (1-\tilde {t}_i^2)^{\frac{\qbar-3}{2}}d\tilde{\tbf} :=\sum_{\kk \geq 0}\lambda_{\kk}Q_{\kk}(\tbf). 
\end{align*}
Lemma \ref{lemma:addition_thm_multi} implies
\begin{align*}
    Q_{\kk}(\langle \x^{(1)} ,\z^{(1)} \rangle ,..,\langle \x^{(i)} ,\z^{(j)} \rangle,..,\langle \x^{({\dbar})} ,\z^{({\dbar})} \rangle) = \frac{|\Sphere^{\qbar-1}|^d}{\prod _{i=1}^{{\dbar}}N(\qbar,k_i)} \sum_{\jj:j_i\in [N(\qbar,k_i)]}Y_{\kk,\jj}(\x)Y_{\kk,\jj}(\z),
\end{align*}
yielding
\begin{align*}
    \kr(\x,\z) =\sum_{\kk \geq 0}\lambda_{\kk}\sum_{\jj:j_i\in [N(\qbar,k_i)]}Y_{\kk,\jj}(\x)Y_{\kk,\jj}(\z).
\end{align*}
Since $\{Y_{\kk,\jj}(\x)\}$ are orthonormal w.r.t.\ the uniform measure in $\ms$ (Lemma \ref{lemma:orth_multi}) we obtain that $\{Y_{\kk,\jj}(\x)\}$ are the eigenfunctions of $\kr^{(L)}$, with the corresponding eigenvalues $\{\lambda_{\kk}  =  |\Sphere^{\qbar-2}|^d \int_{[-1,1]^d} \kr(\tbf) \prod_{i=1}^d Q_{k_i}(t_i)(1-t_i^2)^{\frac{\qbar-3}{2}}d\tbf\}$.
\end{proof}

\subsection{Proof of Lemma \ref{lemma:taylor_to_eigs}}

\begin{lemma}
Let $\kr$ be a multi-dot product kernel with the power series given in \eqref{eq:k_taylor}, where $\x^{(i)},\z^{(i)} \in \Sphere^{\qbar-1}$ respectively are pixels in $\x,\z$.
Then, the eigenvalues $\lambda_{\kk}(\kr)$ of $\kr$ are given by 
$\lambda_{\kk}(\kr) = \left|\Sphere^{\qbar-2}\right|^d \sum_{\s \ge 0} b_{\kk+2\s} \prod_{i=1}^{\dbar}\lambda_{k_i}(t^{k_i+2s_i})$,
where $|\Sphere^{\qbar-2}|$ is the surface area of $\Sphere^{\qbar-2}$, and $\lambda_k(t^n)$ is the $k$'th eigenvalue of $t^n$, given by
\begin{align*}
    \lambda_{k}(t^n) = \frac{n!}{(n-k)!2^{k+1}} \frac{\Gamma \left(\frac{\qbar-1}{2}\right)\Gamma\left(\frac{n-k+1}{2}\right)}{\Gamma\left(\frac{n-k+\qbar}{2}\right)}
\end{align*} 
if $n-k$ is even and non-negative, while $\lambda_k(t^n)=0$ otherwise, and $\Gamma$ is the Gamma function.
\end{lemma}

\begin{proof}
The proof follows the linearity of the integral operator. Let
\begin{align}
    \kr(\x,\z)=\sum_{\n\geq 0}b_{\n} \langle \x^{(1)},\z^{(1)}\rangle ^{n_1}\cdot ...\cdot \langle \x^{(d)},\z^{(d)} \rangle^{n_d},
\end{align}
and denote by $C(\qbar,d)=\left|\Sphere^{\qbar-2}\right|^d$.
Following Lemma \ref{applemma:mercer}  the eigenvalues of $\kr$ are given by
\begin{align*}
        \lambda_{\kk}=& C(\qbar,d)\int_{[-1,1]^d} \kr(\tbf) \prod_{i=1}^d Q_{k_i}(t_i)(1-t_i^2)^{\frac{\qbar-3}{2}}dt_1...dt_d\\
        =&C(\qbar,d)\int_{[-1,1]^d} \sum_{\n \geq 0}b_{\n} \tbf^\n \prod_{i=1}^d Q_{k_i}(t_i)(1-t_i^2)^{\frac{\qbar-3}{2}}dt_1...dt_d\\
        =&C(\qbar,d) \sum_{\n \geq 0}b_{\n} \prod_{i=1}^d \left(\int_{[-1,1]}t^n_iQ_{k_i}(t_i)(1-t_i^2)^{\frac{\qbar-3}{2}}dt_i\right)=C(\qbar,d) \sum_{\n \geq 0}b_{\n}\prod_{i=1}^d\lambda_{k_i}(t^{n_i}).
    \end{align*}
Also note from \cite{azevedo2015eigenvalues} that $\lambda_{k}(t^{n})=0$ whenever $n-k$ is either odd or negative, implying the statement of the lemma.
\end{proof}

%\subsection{Proof of Lemma \ref{lemma:lower_upper_eigs}}

A consequence of the lemma above is that the eigenvalues of a kernel $\kr$ can be bounded by the eigenvalues of other kernels if the power series coefficients of $\kr$ are bounded by the respective coefficients of the other kernels. We summarize this in the following corollary:

\begin{corollary} \label{appcor:lower_upper_eigs}
Let $\kr,\kr^{upper},\kr^{lower}: \ms\rightarrow \Real $ be multi-dot product kernels. Assuming that for $\tbf\in[-1,1]^d$,
\begin{align*}
    \kr(\tbf) = \sum_{\n}b_{\n}\tbf^{\n}\\
    \kr^{upper}(\tbf) = \sum_{\n}b^{upper}_{\n}\tbf^{\n}\\
    \kr^{lower}(\tbf) = \sum_{\n}b^{lower}_{\n}\tbf^{\n}
\end{align*}
  and suppose there exists $\kk_0$ such that for all $\n \ge \kk_0$, $0\leq c_1b^{lower}_{\n}\leq b_{\n}\leq c_2 b^{upper}_{\n}$, with $c_1, c_2 > 0$. Then, for all $\kk \ge \kk_0$, 
\begin{align} \label{eq:lower_upper_eigs}
    c_1\lambda_{\kk}(\kr^{lower})\leq \lambda_{\kk}(\kr)\leq c_2 \lambda_{\kk}(\kr^{upper})
\end{align}
\end{corollary}

This corollary is  an immediate result from Lemma \ref{lemma:taylor_to_eigs}.

\begin{comment}
\begin{proof}
Here we prove the lemma for the upper bound $\kr^{upper}$ in \eqref{eq:lower_upper_eigs}. The proof for the lower bound $\kr^{lower}$ is similar.
We next show that
\begin{align*}
    c_2b^{upper}_{\n}-b_{\n}\geq 0 \Rightarrow c_2 \lambda_{\kk}(\kr^{upper})- \lambda_{\kk}(\kr)\geq 0
\end{align*}
We begin by simple arithmetic that gives us
\begin{align*}
    c_2 \kr^{upper}(\tbf)-\kr(\tbf)=\sum_{\n}(c_2b^{upper}_{\n}-b_{\n})\tbf^\n
\end{align*}
By \cite{azevedo2015eigenvalues} we have that the projection of $t^n$ on the $k$'th Gegenbauer polynomial in $\Sphere^{\qbar-1}$, $Q_k(t)$ is 
\begin{align*}
    \int_{-1}^1 t^n Q_k(t)(1-t^2)^{\frac{\qbar-3}{2}}dt\geq 0
\end{align*}
implying that
\begin{align*}
   c_2& \lambda_{\kk}(\kr^{upper})- \lambda_{\kk}(\kr)=\int_{[-1,1]^l}(c_2 \kr^{upper}(\tbf)-\kr(\tbf))Q_{\kk}(\tbf) \prod_{i=1}^d (1-t_i^2)^{(\qbar-3)/2} d\tbf\\
   =&\int_{[-1,1]^l}\sum_{\n}(c_2b^{upper}_{\n}-b_{\n})t_1^{n_1}\cdot..\cdot t_{\dbar}^{n_{\dbar}}Q_{k_1}(t_1)\cdot...\cdot Q_{k_{\dbar}}(t_{\dbar})\prod_{i=1}^\dbar (1-t_i^2)^{(\qbar-3)/2}dt_1..dt_{\dbar}\\
   =&\sum_{\n}(c_2b^{upper}_{\n}-b_{\n})\left(\int_{[-1,1]}t_1^{i_1}Q_{k_1}(t_1)(1-t_1^2)^{\frac{\qbar-3}{2}}dt_1\right)\cdot...\cdot \left(\int_{[-1,1]}t_{\dbar}^{i_{\dbar}}Q_{k_{\dbar}}(t_{\dbar})(1-t_{\dbar}^2)^{\frac{\qbar-3}{2}}dt_{\dbar}\right)\geq 0
\end{align*}
This concludes that $c_2 \lambda_{\kk}(\kr^{upper})- \lambda_{\kk}(\kr)\geq 0$
\end{proof}
\end{comment}

\section{Factorizable kernels}

In this section we prove results presented in Section~\ref{sec:factorizable}. We prove Theorem~\ref{thm:taylor_to_eigs}, which determines the eigenvalues of factorizable kernels whose power series coefficients decay at a polynomial rate. The following supporting lemma proves the theorem for $d=1$.

\begin{lemma}\label{lemma:from_taylor_to_eigs}
Let $\tilde \kappa(t) = \sum_{n=0}^\infty \tilde a_n t^n$ where $\tilde a_n=O(n^{-\nu})$ with $\nu>1$ and not integer. Then, the eigenvalues of $\tilde \kappa$ w.r.t.\ the uniform measure in $\Sphere^{\qbar-1}$ are
\begin{align*}
    \lambda_k=\Theta\left(k^{-(\qbar+2\nu -3)}\right).
\end{align*}
\end{lemma}

\begin{proof}
By applying Corollary \ref{appcor:lower_upper_eigs} with $\dbar=1$   we have that if $f(t)=\sum_{n=0}^\infty a_nt^n$ and $g(t)=\sum_{n=0}^\infty b_nt^n$ with $c_1a_n\leq b_n \leq c_2a_n$ then it holds that $\lambda_k(g)=\Theta (\lambda_k(f))$. It is therefore enough to find $f(t)=\sum_{n=0}^\infty \tilde a_n t^n$ where $\tilde a_n=O(n^{-\nu})$ and then calculate its eigenvalues. By \cite{flajolet2009analytic} (Thm.~VI.1, page 381), the function $f(t)=(1-t)^{\nu-1}$, where $\nu >1$ is non-integer, satisfies $f(t)=\sum_{n=0}^\infty \tilde a_n t^n$ with $\tilde a_n=O\left(n^{-\nu}\right)$. Moreover, according to \cite{bietti2020deep} (Thm.~7, page 17), the eigenvalues of $f(t)=(1-t)^{\nu-1}$ in $\Sphere^{\qbar-1}$ are
\begin{align*}
    \lambda_k(f)=c_1k^{-(\qbar+2\nu -3)},
\end{align*}
which concludes the proof.
\end{proof}

Relying on the lemma, we can now prove Theorem~\ref{thm:taylor_to_eigs}.
\begin{theorem}
\label{thmapp:taylor_to_eigs}
Let $\kr$ be a factorizable multi-dot product kernel, and let $\R \subseteq [d]$ denote its receptive field. Suppose that $\kr$ can be written as a multivariate power series, $\kr(\tbf)=\sum_{\n \ge 0} b_\n \tbf^\n$ with
\begin{align*}
    b_\n \sim %c \n^{-\nu} = 
    c \prod_{i \in \R, \, n_i > 0} n_i^{-\nu}.
\end{align*}
with constants $c>0$,non-integer $\nu > 1$, and $b_\n=0$  if $n_i >0$ for any $i \not \in \R$. Then the eigenfunctions of $\kr$ w.r.t. the uniform measure are the SH-products, and its eigenvalues $\lambda_\kk(\kr)$ satisfy
\begin{align*}
    \lambda_\kk \sim 
    %c \kk^{-(\qbar+2\nu-3)} = 
    \tilde c \prod_{i \in \R, \, k_i>0}  k_i^{-(\qbar+2\nu-3)},
\end{align*}
where $\kk \in \N^d$ be a vector of frequencies. Finally, $\lambda_\kk=0$ if $k_i>0$ for any $i \not\in \R$.
\end{theorem}

\begin{proof}
Since $\kr(\tbf)$ is factorizable and can be written by a power series it can be written as
\begin{align*}
    \kr(\tbf) = c \tilde \kappa(t_1) \cdot ... \cdot \tilde \kappa(t_d),
\end{align*}
where $\tilde \kappa(t) \sim \sum_{n=0}^\infty n^{-\nu} t^n$, and it can be readily shown that 
\begin{align*}
    \lambda_{\kk} (\kr) &= c\lambda_{k_1}(\tilde \kappa) \cdot ... \cdot \lambda_{k_{d}}(\tilde \kappa).
\end{align*}
Using Lemma \ref{lemma:from_taylor_to_eigs} we have that \begin{align*}
    c\lambda_{k_1}(\tilde\kappa)\cdot ..\cdot \lambda_{k_{\dbar}}(\tilde \kappa) \sim \tilde c \prod_{i \in R, k_i>0} k_i^{-(\qbar+2 \nu -3 )},
\end{align*}
which concludes our proof.
\end{proof}

\begin{comment}

\begin{proof}
Since $\kr(\tbf) = \sum_{\n \ge 0} b_\n \tbf^\n$ where $b_{\n}= c(n_1\cdot..\cdot n_{\dbar})^{- \nu}$ \mg{note that we use here equality} it is clear that  %factorizable:
\begin{align*}
    \kr(\tbf)&=
    %\sum_{\n}  b_{\n}\tbf^{\n}=
    c\sum_{\n}  (n_1\cdot...\cdot n_{\dbar})^{- \nu}t_1^{n_1}\cdot..\cdot t_{\dbar}^{n_{\dbar}}
    = c\left(\sum_{n_1 \geq 0} n_1^{-\nu} t_1^{n_1}\right)\cdot ...\cdot \left(\sum_{n_d \geq 0}n_d^{- \nu}t_d^{n_d}\right).
\end{align*}
Denote by $\tilde \kappa(t) = \sum_{n=0}^\infty n^{-\nu} t^n$ then
\begin{align*}
    \kr(\tbf) = c \tilde \kappa(t_1) \cdot ... \cdot \tilde \kappa(t_d),
\end{align*}
and the eigenvalues satisfy \mg{do we need to justify this step?}
\begin{align*}
    \lambda_{\kk} (\kr) &= c\lambda_{k_1}(\tilde \kappa) \cdot ... \cdot \lambda_{k_{d}}(\tilde \kappa).
\end{align*}
Using Lemma \ref{lemma:from_taylor_to_eigs} we have that \begin{align*}
    c\lambda_{k_1}(\tilde\kappa)\cdot ..\cdot \lambda_{k_{\dbar}}(\tilde \kappa)= \tilde c  (k_1\cdot \ldots \cdot k_{d})^{-(\qbar+2  \nu -3 )},
\end{align*}
which concludes our proof. \mg{note that we do not treat explicitly the zero indices}
\end{proof}

\end{comment}

\section{Positional bias of eigenvalues}

We next prove results presented in Section~\ref{sec:spatial}. We next  prove Theorem \ref{thm:hierarchy}.

\begin{theorem}
Let $\kr^{(L)} $ be hierarchical and factorizable of depth $L>1$ with filter size $q$, so that $\kr^{(L)}(\tbf)=\sum_{\n \ge 0} b_\n \tbf^\n=c\sum_{\n \ge 0} a_{n_1}\cdot..\cdot a_{n_d} \tbf^\n$ with $a_0>0$ and $a_{n_i}= n_i^{-\nu}$ for $\nu>1$. Then there exist a scalar $A=1+\frac{1}{a_0}$ such that:
\begin{enumerate}
    \item The power series coefficients of $\kr^{(L)}$ satisfy
    \begin{align*}
        c_{A,\n}\n^{-\nu}\leq  b_\n,
        %\leq c B^{\bar p_\n^{(L)}}\n^{-\nu}.
    \end{align*} 
    where 
    \begin{align*}
        c_{A ,\n}   &= c_L \prod_{i=1}^d A^{\min(p^{(L)}_i,n_i)}.
    \end{align*}
    \item The eigenvalues $\lambda_\kk(\kr^{(L)})$ satisfy
    \begin{align*}
        c_{A ,\kk}  \prod_{\substack{i=1\\n_i>0}}^d k_i^{-(\qbar+2\nu-3)} \le \lambda_\kk,
    \end{align*}
    where
    \begin{align*}
        c_{A ,\kk}   &= \tilde c_L \prod_{i=1}^d A^{\min(p^{(L)}_i,k_i)}.
    \end{align*}
\end{enumerate}
$c_L$ and $\tilde c_L$ are constants that depends on $L$, and $p^{(L)}_i$ denotes the number of paths from pixel $i$ to the output of $\kr^{(L)}$. 
\end{theorem}
% \begin{theorem}
% Let $\kr^{(L)}$ be a factorizable CNTK or CGPK EqNet of depth $L$, filter size $q$ where $\kr^{(L)}(\tbf)=\sum_{\n \ge 0} b_\n \tbf^\n=c\sum_{\n \ge 0} a_{n_1}\cdot..\cdot a_{n_d} \tbf^\n$ with $a_0>0$ and $a_{n_i}= n_i^{-\nu}$ for $\nu>1$. Then
% \begin{enumerate}
%     \item The Taylor coefficients satisfy\begin{align*}
%     c A^{\bar p_\n^{(L)}}\n^{-\nu}\leq  b_\n \leq c B^{\bar p_\n^{(L)}}\n^{-\nu}.
%     \end{align*} 
%     \item The eigenvalues satisfy \begin{align*}
%     \tilde c A^{\bar p_\n^{(L)}}\kk^{-(\qbar+2\nu-3)} \le \lambda_\kk \le \tilde c B^{\bar p_\n^{(L)}}\kk^{-(\qbar+2\nu-3)} ,
% \end{align*}
% % where
% % \begin{align*}
% %     c_j(\sgn(\kk)) &= \prod_{i=1}^d c_j(p_i^{(L)},k_i), ~j=1,2
% % \end{align*}
% % with
% % \begin{align*}
% %     c_j(p_i^{(L)},k_i) &= \begin{cases}
% %     c^{1/d}, & k_i=0\\ c^{1/d}A_j^{p_i^{(L)}} & k_i \ge 1, ~j=1,2
% %     \end{cases}\\
% % \end{align*}
% \end{enumerate}
% where $\bar p^{(L)}_{\n}=\sum_{j=1}^d \sgn(n_j) (p^{(L)}_j-1)$ with both $A,B>1$. $p^{(L)}_j$ denotes the number of paths from pixel $j$ to the output of the corresponding equivariant network.
% \end{theorem}
To prove the theorem we provide several supporting lemmas and the following definition:

\begin{comment}
\begin{proposition}
Let $N_l>0$ and $ l\cdot N_l<n$. Define $I_1=\{(k_1,..,k_l)|k_1+..+k_l=n\}$ and $I_2=\cup_{i=1}^{l} J_i(N_l)$ such that $J_i(N_l)=\{(k_1,..,k_l)| k_1+..+k_l=n, i=argmax(k_1,..,k_l), k_i\geq N_l \}$, then $I_1=I_2$
\end{proposition}
\textbf{Remark} if $i=\mathrm{argmax}(k_1,..,k_l)$ is not unique we choose the minimal $i$
\begin{proof}
Let $(k_1,,.,k_l)\in I_1$, then there exists a unique (up to the above remark) $i=argmax(k_1,..,k_l)$ and since $ l\cdot N_l<n$ it must hold that at least one $k_i>N_l$ so in particular $max(k_1,..,k_l)>N_l$. \end{proof}
\end{comment}

\begin{definition}
A kernel $\tilde\kr^{(L)} [-1,1]^{q^L} \rightarrow \Real$ is called stride-q hierarchical of depth $L>1$ if there exists a sequence of kernels $\tilde\kr^{(1)},...,\tilde\kr^{(L)}$ such that $\tilde\kr^{(l)}(\tbf)=f\left(\tilde\kr^{(l-1)}(\tbf_1),...,\tilde\kr^{(l-1)}(\tbf_q)\right)$ with $f:\Real^q \rightarrow \Real$, $\tbf=(\tbf_1,...,\tbf_q) \in [-1,1]^{q^{l-1}}$ and $\kr^{(1)}(t)=t \in [-1,1]$. A kernel $\kr^{(L)}:[-1,1]^{q(L-1)+1} \rightarrow \Real$ is stride-1 hierarchical if for all $1<l\le L$, $\kr^{(l)}=f\left(\kr^{(l-1)}(\tbf_1),\kr^{(l-1)}(s_1\tbf_1),...,\kr^{(l-1)}(s_{q-1}\tbf_1)\right)$ and $\tbf_1 \in [-1,1]^{q(l-2)+1}$.
\end{definition}

We next formulate the relation between the power series coefficient of the two kernels:

\begin{lemma}\label{lemma:cgpk_with_stride}
Let $\kr^{(L)}(\tbf):[-1,1]^d \rightarrow \Real$ be stride-1 kernel and $\tilde \kr ^{(L)}(\tilde \tbf):[-1,1]^{q^{L}} \rightarrow \Real$ be stride-$q$ kernel.  Then, there exists a variables substitution $S:[q^{L}]\rightarrow [d] $ such that if $\tilde t_{S(j)}=t_{j}$ for all $j \in [q^L]$ then 
\begin{align*}
    \tilde \kr ^{(L)}( t_{S(0)},.., t_{S(q^{L}-1)})\equiv \kr^{(L)}(t_0,..,t_{d-1}).
\end{align*}
Moreover, if $\kr^{(L)}(\tbf)=\sum_{\n\ge 0}b_\n\tbf^\n$ and $\tilde \kr ^{(L)}(\tilde \tbf)=\sum_{\tilde \n\ge 0}\tilde b_{\tilde \n}\tilde \tbf^{\tilde \n}$ then
\begin{align*}
    b_\n=\sum_{\mathcal{S}}\tilde b_{\tilde \n}
\end{align*}
where $\mathcal{S}=\{\tilde n_0,..,\tilde n_{q^{L}-1}|\forall i=0,..,d-1, \sum_{i=S(j)}\tilde n_j=n_i\}$.

% kernels with patch size $q$ ReLU activation and strides $\sigma_1^{(1)},\sigma_1^{(2)},..\sigma_1^{(L-1)}$ and $\sigma_2^{(1)},\sigma_2^{(2)},..\sigma_2^{(L-1)}$. Assume that for all $l\ge 1$ $\sigma_1^{(l)}=\sigma_2^{(l)}$ and that $\exists l'$ such that $\sigma_1^{(l')}=1$ and $\sigma_2^{(l')}=l(q-1)+1$. Then, both kernels can be written as $\tilde \kr_i ^{(L)}(\tbf)=\sum_{\n\ge0}b^{(i)}_{\n}\tbf^\n$ where  
% \begin{align*}
%     b^{(2)}_{\n}=c(\sgn(\n))b^{(1)}_{\n}
% \end{align*}
% with $\le c(\sgn(\n))\le $
\end{lemma}
\begin{proof}

We construct the mapping $S$ and prove its correctness by induction. For any index $j=0,..,q^{L}-1$ we  write $j=a_{L-1}q^{L-1}+a_{L-2}q^{L-2}+..+a_1q+a_0$ where $a_i=0,1,..,q-1$. Then, we define $S(j):=S_{L}(j)=a_{L-1}+a_{L-2}+..+a_0$. We next prove by induction that $\tilde \kr ^{(L)}( t_{S(0)},.., t_{S(q^{(L)}-1)})\equiv \kr^{(L)}(t_0,..,t_{d-1})$. For $L=2$ we have:
\begin{align*}
    \tilde \kr^{(2)}(t_{S(0)},..,t_{S(q^2-1)})&= f\left( \kr^{(1)}(t_{S(0)},..,t_{S(q-1)}),...,\kr^{(1)}(t_{S((q-1)q)},..,t_{S(q^2-1)})\right)\\
    &=f\left( \kr^{(1)}(t_{0},..,t_{q-1}),...,\kr^{(1)}(t_{q-1},..,t_{2q-2})\right)=f\left( \kr^{(1)}(\tbf),\kr^{(1)}(s_1\tbf),...,\kr^{(1)}(s_{q-1}\tbf)\right)
\end{align*}
Where $\tbf = t_0,..,t_{q-1}$ and $s$ is the shift operator. This concludes the case of $L=2$. For $L>2$ we assume that $S_{L-1}(j)=a_{L-2}+..+a_0$ is the correct assignment for $q^{L-1}-1$ variables and get that 
\begin{align*}
    \tilde \kr^{(L)}(t_{S(0)},..,t_{S(q^{L}-1)})&= f\left( \tilde\kr^{(L-1)}\left(t_{S(0)},..,t_{S((q-1)q^{L-2})+..+q-1)}\right),...,\tilde\kr^{(L-1)}\left(t_{S((q-1)q^{L-1})},..,t_{S((q-1)q^{L-1}+..+q-1)}\right)\right)\\
    &=f\left( \tilde\kr^{(L-1)}\left(t_{0},..,t_{(L-1)(q-1)}\right),...,\tilde\kr^{(L-1)}\left(t_{(q-1)},..,t_{L(q-1))}\right)\right)\\
    &=^{(1)}f\left( \kr^{(L-1)}\left(\tbf\right),...,\kr^{(L-1)}\left(s_{q-1}\tbf\right)\right),
\end{align*}
where $^{(1)}$ holds from the induction hypothesis and $\tbf=t_0,..,t_{(L-1)(q-1)}$.  

Finally since $f$ is an analytic function it holds that:
\begin{align*}
    \kr^{(L)}(\tbf)=\tilde \kr^{(L)}(  t_{S(0)},.., t_{S(q^L-1)})=\sum_{\tilde \n\ge 0}\tilde b_{\tilde \n}  t_{S(0)}^{\tilde n_1}\cdot ..\cdot t_{S(q^{L}-1)}^{\tilde n_{q^{L}-1}}=\sum_{\n\ge 0}\tbf^{\n}\sum_{\mathcal{S}}\tilde b_{\tilde \n}
\end{align*}
where $\mathcal{S}=\{\tilde n_0,..,\tilde n_{q^{L}-1}|\forall i=0,..,d-1\sum_{i=S(j)}\tilde n_j=n_i\}$. Therefore, from the uniqueness of the power series we get that 
\begin{align*}
    b_\n=\sum_{\mathcal{S}}\tilde b_\n.
\end{align*}
\end{proof}

\begin{lemma} \label{lemma:variables_aggregation}
Let $\kk \in \N^{m}$ and consider the series $S_m(n) =\sum_{k_1+\ldots+k_m = n} \prod_{i=1}^m k_i^{-\nu} = \sum_{|\kk|=n} \kk^{-\nu}$ with $\nu>1$ and the convention $0^{-\nu}=a_0>0$. Then, for $n\ge m$, $S_m(n)$ is bounded from above and below as follows
\begin{equation}
   A^{m-1} n^{-\nu} \leq S_m(n) \leq B^{m-1} n^{-\nu},
\end{equation}
with $B>A=(a_0+1)>1$  constants.
\end{lemma}
\begin{proof}
We show this by induction over $m$, i.e., the length of the vector $\kk$. We begin by showing this for $S=S_2(n)$ for any $n \geq 2$, i.e., 
$A n^{-\nu} \leq S = \sum_{k=0}^{n} k^{-\nu}(n-k)^{-\nu} \leq B n^{-\nu}$ for constants $A$ and $B$.

\textbf{Lower bound.} For  $n > 2$
it holds that 
\begin{align*}
    S = \sum_{k=0}^{n} k^{-\nu}(n-k)^{-\nu}&=2\cdot a_0\cdot n^{-\nu}+2\cdot (n-1)^{-\nu}+ \sum_{k=2}^{n-2} k^{-\nu}(n-k)^{-\nu} \\  & \geq 2\cdot a_0\cdot n^{-\nu}+2\cdot (n-1)^{-\nu} \\ 
    & \geq  2(a_0+1)n^{-\nu} \geq (2a_0 + 1) n^{-\nu} \\
    & \geq (a_0+1) n^{-\nu}.
\end{align*}
For $n=2$, we have that $S=2 a_0 n^{-\nu}+(n-1)^{-\nu} \geq (2 a_0 + 1)n^{-\nu} \geq (a_0+1) n^{-\nu}$.

Therefore, it holds for $n \geq 2$ that  $S_2(n) \geq A_ n^{-\nu}$, where $A  = a_0 + 1$.

%\mg{I changed to inequality which holds also for $n=2$, for $n>2$ the bound can be tighter $2(c_0+1)n^{-\nu}$. The equality was $(c_0+1)n^{-\nu}$. In any case, the bounds are not true for $n=1$. }\ag{It was $(c_0+1)$ to be consistent with the induction hypothesis but this is not very important }

\textbf{Upper bound.}
We show that for $n \geq 2$ it holds that  $n^\nu S_2(n)=n^\nu\sum_{k=0}^{n}k^{-\nu}(n-k)^{-\nu}\leq (2a_0+2)+ \frac{2^{(\nu+1)}}{\nu-1}$. This follows from:
\begin{align*}
    &n^\nu\sum_{k=0}^{n}k^{-\nu}(n-k)^{-\nu}\le (2a_0+2^{\nu+1})+\sum_{k=2}^{n-2}\left(\frac{n-k+k}{k(n-k)}\right)^{\nu}
    =(2a_0+2^{\nu+1})+\sum_{k=2}^{n-2}\left(\frac{n-k}{k(n-k)}+\frac{k}{k(n-k)}\right)^{\nu}\\
    &=(2a_0+2^{\nu+1})+\sum_{k=2}^{n-2}\left(\frac{1}{k}+\frac{1}{(n-k)}\right)^{\nu} 
    \leq (2a_0+2^{\nu+1})+ \sum_{k=2}^{n-2} \left(2\max\left\{\frac{1}{k},\frac{1}{n-k}\right\}\right)^{\nu}\le (2a_0+2^{\nu+1})+2^\nu 2\sum_{k=2}^{n/2}k^{-\nu}.
\end{align*}

Note that $f(k) = k^{-\nu}$ is monotonically decreasing and therefore can be bounded by the integral 
\begin{align*}
    \sum_{k=2}^{n/2}k^{-\nu}\leq \int_{1}^{n/2}\frac{1}{x^\nu} dx =\frac{1}{\nu-1}-\left(\frac{2}{n}\right)^{\nu-1}\frac{1}{\nu-1}\leq \frac{1}{\nu-1}
\end{align*}
So overall we have that $n^\nu S_2(n) \leq 2a_0+2^{\nu+1}+ \frac{2^{(\nu+1)}}{\nu-1}$ implying that $S_2(n)\leq B n^{-\nu}$ with $B=2a_0+2^{\nu+1}+ \frac{2^{(\nu+1)}}{\nu-1}$.

\textbf{Induction step.}
We next use induction to prove the lemma for $S_{m}(n)$ for $m>2$ and $n\ge m$.
Assume the lemma holds for $S_m$, i.e., $A^{m-1} n^{-\nu} \leq S_m(n) \leq B^{m-1} n^{-\nu}$ for $n \geq m$ and $A=a_0+1 > 1$, we aim to prove this for $S_{m+1}(n)$ for $n \geq m+1$.
\begin{align*}
    S_{m+1}(n) = \sum_{k_1=0}^{n} k_1^{-\nu} \sum_{k_2+...+k_{m+1} = n - k_1}  k_2^{-\nu} \cdots k_{m+1}^{-\nu}.
\end{align*}
Using the induction assumption, we obtain 
\begin{align*}
     S_{m+1}(n) &= \sum_{k_1=0}^{n} k_1^{-\nu} \sum_{k_2+...+k_{m+1} = n - k_1}  k_2^{-\nu} \cdots k_{m+1}^{-\nu}\\
     & \ge a_0 \sum_{k_2+...+k_{m+1} = n}  k_2^{-\nu} \cdots k_{m+1}^{-\nu}+\sum_{k_2+...+k_{m+1} = n - 1}  k_2^{-\nu} \cdots k_{m+1}^{-\nu} \\
     &\ge a_0 A^{m-1}n^{-\nu}+A^{m-1}(n-1)^{-\nu}\ge (a_0+1)^{m}n^{-\nu} = A^m n^{-\nu}.
\end{align*}
Note that in the two sums above the induction assumption holds since $n \geq n-1\geq m$.
This concludes the proof for the lower bound.
The proof for the upper bound proceeds in a similar way.
\end{proof}
 \begin{lemma} \label{lemma:variables_aggregation_nlem}
Let $\kk \in \N^{m}$ and consider the series $S_m(n) = \sum_{|\kk|=n} \kk^{-\nu}$ with $\nu>1$ and the convention $0^{-\nu}=a_0>0$. Then, for $2\le n \le m$, $S_m(n)$ is bounded from above and below as follows. 
\begin{equation}
   a_0^{m-n}A^{n-1} n^{-\nu} \leq S_m(n) \leq a_0^{m-n}\left(\frac{m\cdot e }{n} \right)^{n}B^{n-1} n^{-\nu},
\end{equation}
where $A,B$ are given in Lemma~\ref{lemma:variables_aggregation}. Note that for $m=n$ the lower bound boils down to the lower bound in Lemma ~\ref{lemma:variables_aggregation}. 
\end{lemma}

\begin{proof} We next prove the lemma for $2\le n\leq m$. 
\begin{align*}
    S_{m}(n)=\sum_{k_1+...+k_m=n}k_1^{-\nu}\cdot ... \cdot k_m^{-\nu} \ge a_0^{m-n}\sum_{k_1+...+k_n=n}k_1^{-\nu}\cdot ... \cdot k_n^{-\nu} \ge a_0^{m-n} A^{n-1}n^{-\nu},
\end{align*}
where the last inequality holds from Lemma~\ref{lemma:variables_aggregation} with $n=m$.

For the upper bound we have 
\begin{align*}
    S_{m}(n) &= \sum_{k_1+..+k_m=n}k_1^{-\nu} \cdot ... \cdot k_m^{-\nu}\le^{(1)} a_0^{m-n}{m\choose n} \sum_{k_1+..+k_n=n} k_1^{-\nu} \cdot ... \cdot k_n^{-\nu} \\
    &\le^{(2)} a_0^{m-n} {m\choose n} B^{n-1}n^{-\nu} \le a_0^{m-n} \left(\frac{m\cdot e }{n}\right)^{n}B^{n-1} n^{-\nu},
\end{align*}
where $^{(1)}$ considers subsets of size $n$ and sets the remaining orders $k_i$ to zero. Note that since $n\le m$ this covers all the options of satisfying the sum $k_1+..+k_m=n$ (with some repetitions). $^{(2)}$ uses the bound of Lemma \ref{lemma:variables_aggregation}.
\end{proof}

\begin{lemma}\label{lemma:factor_taylor_const}
Let $\kr^{(L)}$ be an hierarchical factorizable kernel of depth $L$ and filter size $q$, where $\kr^{(L)}(\tbf)=\sum_{\n \ge 0} b_\n \tbf^\n=c\sum_{\n \ge 0} a_{n_1}\cdot..\cdot a_{n_d} \tbf^\n$ with $a_0>0$ and $a_{n_i}= n_i^{-\nu}$ for $\nu>1$. Then,  the Taylor coefficients of  $\kr^{(L)}$ satisfy
\begin{align*}
    c_{A ,\n}\n^{-\nu}\leq  b_\n \leq c_{B ,\n}\n^{-\nu}
    %\leq c B^{\bar p_\n^{(L)}}\n^{-\nu}.
\end{align*} 
where 
\begin{align*}
    c_{A ,\n}   &= c_L \prod_{i=1}^d A^{\min(p^{(L)}_i,n_i)}
\end{align*}
and $c_{B,\n}=\bar c_L \prod_{i=1}^{d}c_B(p^{(L)}_i,n_i)$
\begin{align*}
    c_B(p^{(L)}_i,n_i)  &=  \begin{cases}
 \left(\frac{p^{(L)}_i\cdot e }{n_i}\right)^{n_i}B^{n_i}, & 1\leq n_i<p^{(L)}_i\\
B^{p^{(L)}_i}, & n_i\ge p^{(L)}_i
\end{cases}
\end{align*} 
 with  $B\ge A=1+\frac{1}{a_0}$ and $c_L,\bar c_L$ are constants. $p^{(L)}_i$ denotes the number of paths from pixel $j$ to the output of the corresponding equivariant network.
\end{lemma}

\begin{proof}
Since $\kr^{(L)}$ is factorizable we can use the hierarchical stride $q$ kernel $\tilde \kr ^{(L)}(\tilde \tbf)$ and write:
\begin{align*}
    \tilde \kr^{(L)}(\tilde \tbf)=\sum_{\tilde \n\ge 0}\tilde b_{\tilde \n}  \tbf^{\tilde \n}=\sum_{\tilde \n\ge 0}a_{\tilde n_1}\cdot ..\cdot a_{\tilde n_{q^L-1}}\tbf^{\tilde \n}
\end{align*}
with $a_{\tilde n_i}=\tilde n_i^{-\nu}$.
Moreover using the mapping $S$ from lemma \ref{lemma:cgpk_with_stride} we have that   $\kr^{(L)}(\tbf)=\sum_{\n\ge0}b_\n\tbf^\n$ with 
\begin{align*}
    b_\n=\sum_{\mathcal{S}}\tilde b_{\tilde \n}=c\sum_{\mathcal{S}}\tilde \n ^{-\nu}
\end{align*}
where $\mathcal{S}=\{\tilde n_1,..,\tilde n_{q^{L-1}}|\forall i=1,..,d, \sum_{i=S(j)}\tilde n_j=n_i\}$. Note that $|\{i|S(i)=j\}|=p^{(L)}_j$ where $p^{(L)}_j$ denotes the number of paths from the input pixel to the output, therefore by combining Lemma \ref{lemma:variables_aggregation} for the case of $p^{(L)}_j \ge n_j$ and Lemma \ref{lemma:variables_aggregation_nlem} for the case of $p^{(L)}_j \le n_j$ we have that 
\begin{align*}
    \tilde c_{A,\n}\n^{-\nu}\leq  b_\n 
\end{align*} 
where $c_{A,\n}=\prod_{i=1}^{d}c(p^{(L)}_i,n_i)$ and 
\begin{align*}
    c(p^{(L)}_i,n_i)  &=  \begin{cases}
a_0^{p^{(L)}_i-n_i}(1+a_0)^{n_i-1}, & n_i<p^{(L)}_i\\
(1+a_0)^{p^{(L)}_i-1}, & n_i\ge p^{(L)}_i
\end{cases}
\end{align*} 
So all in all we get 
\begin{align*}
    c(p^{(L)}_i,n_i) :=  (1+a_0)^{-1} a_0^{p_i^{(L)}}A^{\min(p^{(L)}_i,n_i)}
\end{align*}
with $A=1+\frac{1}{a_0}$. This leads to
\begin{align*}
    c_{A ,\n}   &= c_L \prod_{i=1}^d A^{\min(p^{(L)}_i,n_i)}
\end{align*}
where $A=1+\frac{1}{a_0}$ and  $c_L=(1+a_0)^{-d}\cdot a_0^{\sum_{i=1}^{d}p^{(L)}_i}$. The same set of steps using lemmas \ref{lemma:variables_aggregation} and \ref{lemma:variables_aggregation_nlem} leads to the results of $c_{B ,\n}$
\end{proof}

\begin{lemma}\label{thm:cgpk_eqnet_eig_coefs}
Let $\kr^{(L)}$ be a stride-1 hierarchical and factorizable of depth $L$ and filter size $q$, where $\kr^{(L)}(\tbf)=\sum_{\n \ge 0} b_\n \tbf^\n=c\sum_{\n \ge 0} a_{n_1}\cdot...\cdot a_{n_d} \tbf^\n$ with $a_0>0$, and $a_{n_i}= n_i^{-\nu}$ for $\nu>1$. Then, 
the eigenvalues $\lambda_\kk$ of $\kr^{(L)}$ satisfy
\begin{align*}
     \lambda_\kk \ge  c_{A ,\kk}  \prod_{\substack{i=1\\n_i>0}}^d k_i^{-(\qbar+2\nu-3)} 
    \end{align*}
\begin{align*}
    c_{A ,\kk}   &= c_L \prod_{i=1}^d A^{\min(p^{(L)}_i,k_i)},
\end{align*}
 with $A=1+\frac{1}{a_0}$ and $p^{(L)}_i$ denotes the number of paths from pixel $i$ to the output of the corresponding equivariant network.
\end{lemma}

\begin{proof}
  Using Lemma~\ref{lemma:factor_taylor_const} we have \begin{align*}
    b_{\n} \ge & c \prod_{i=1}^d A^{\min(p_i,n_i)} n_i^{-\nu}.
\end{align*}
Using Lemma \ref{lemma:taylor_to_eigs} we have
\begin{align*}
    \lambda_{\kk} =& |\Sphere^{ \qbar-2}|^d \sum_{\s \ge 0} b_{\kk+2\s} \lambda_{\kk}\left( \tbf^{\kk+2\s} \right),
\end{align*}
where we denote by $\lambda_{\kk}\left( \tbf^{\kk+2\s} \right)=\prod_{i=1}^d \lambda_{k_i} \left( t_i^{k_i+2s_i} \right)$. This implies that
\begin{align*}
    \lambda_{\kk} \ge & c|\Sphere^{ \qbar-2}|^d \sum_{\s \ge 0} \prod_{i=1}^d A^{\min(p_i,k_i+2s_i)} (k_i+2s_i)^{-\nu} \lambda_{k_i}\left( t_i^{k_i+2s_i} \right).
\end{align*}
Applying the distributive law 
\begin{align*}
    \lambda_{\kk} \ge & c|\Sphere^{ \qbar-2}|^d \prod_{i=1}^d \sum_{s_i \ge 0} A^{\min(p_i,k_i+2s_i)} (k_i+2s_i)^{-\nu} \lambda_{k_i}\left( t_i^{k_i+2s_i} \right) = \prod_{i=1}^d \lambda_{k_i}(\kr_i),
\end{align*}
where we define the kernel $\kr_i(t)$ by the power series 
\begin{align*}
    \kr_i(t)=\sum_{n_j=0}^\infty c^{1/d} A^{\min(p_i,n_j)} n_j^{-\nu} t^{n_j}.
    %= c^{1/d} + A^{p_i} \sum_{n=1}^\infty c^{1/d} n^{-\nu} t^n.
\end{align*}
Therefore,
\begin{align*}
    \lambda_{\kk} &\ge c \prod_{i=1}^d \left(   \sum_{n_i=0}^\infty A^{\min(p_i,n_i)} n_i^{-\nu} \lambda_{k_i}\left(t_i^{n_i}\right)\right)= c \prod_{i=1}^d \left(   \sum_{s_i= 0}^\infty A^{\min(p_i,k_i+2s_i)} (k_i+2s_i)^{-\nu} \lambda_{k_i}\left(t^{k_i+2s_i}\right)\right)\\
    \ge & c \prod_{i=1}^d A^{\min(p_i,k_i) } \left( \sum_{s_i=0}^\infty  (k_i+2s_i)^{-\nu} \lambda_{k_i}\left(t^{k_i+2s_i}\right)\right).
    % =& c \prod_{i=1}^d \left(   \sum_{n=0}^\infty \left(\frac{A}{n_0}\right)^{p_i-1} n^{-\nu} \lambda_{k_i}\left(t^n\right)-\sum_{n=0}^{p_i} \left(\frac{A}{n_0}\right)^{p_i-1} n^{-\nu} \lambda_{k_i}\left(t^n\right)+\sum_{n=0}^{p_i} \left(\frac{A}{n_0}\right)^{n_i-1} n^{-\nu} \lambda_{k_i}\left(t^n\right)\right)\\
    % =& c \prod_{i=1}^d \left(\frac{A}{n_0}\right)^{p_i-1} \left(   \sum_{n=0}^\infty  n^{-\nu} \lambda_{k_i}\left(t^n\right)-\sum_{n=0}^{p_i}\left(1- \left(\frac{A}{n_0}\right)^{n_i-p_i} \right)n^{-\nu} \lambda_{k_i}\left(t^n\right)\right)\\
    % \geq& c \prod_{i=1}^d  \left(\frac{A}{n_0}\right)^{p_i-1}\lambda_{k_i}(\kr_i)\left( 1-\frac{p_i!}{(p_i-k_i)!2^{k_i+1}\lambda_{k_i}(\kr_i)}\frac{p_i^{-\nu}\Gamma \left(\frac{\qbar-1}{2}\right)\Gamma\left(\frac{p_i-k_i+1}{2}\right)}{\Gamma\left(\frac{p_i-k_i+\qbar}{2}\right)}\sum_{n=0}^{p_i}\left(1- \left(\frac{A}{n_0}\right)^{n_i-p_i} \right)\right)\\
    % \geq& c \prod_{i=1}^d  \left(\frac{A}{n_0}\right)^{p_i-1}\lambda_{k_i}(\kr_i)\left(1-\frac{p_i!}{(p_i-k_i)!2^{k_i+1}\lambda_{k_i}(\kr_i)}\frac{p_i^{-\nu+1}\Gamma \left(\frac{\qbar-1}{2}\right)\Gamma\left(\frac{p_i-k_i+1}{2}\right)}{\Gamma\left(\frac{p_i-k_i+\qbar}{2}\right)}\right)\\
    % :=&c \prod_{i=1}^d  \left(\frac{A}{n_0}\right)^{p_i-1}\lambda_{k_i}(\kr_i)\left(1-C(k_i,p_i)\right)
\end{align*}
Therefore, using Theorem \ref{thm:taylor_to_eigs} we get that 
\begin{align*}
     \lambda_\kk \ge  c_{A ,\kk}  \prod_{\substack{i=1\\n_i>0}}^d k_i^{-(\qbar+2\nu-3)}  
    \end{align*}
\begin{align*}
    c_{A ,\kk}   &= c_L \prod_{i=1}^d A^{\min(p^{(L)}_i,k_i)}.
\end{align*}
\end{proof}

\section{Kernels associated with the equivariant network}

In this section we prove Theorem~\ref{thm:eqnet-eigs}
%Theorems \ref{thm:cgpk-eqnet-eigs} and \ref{thm:cntk-eqnet-eigs}
presented in Section~\ref{sec:equivariant}.

\begin{theorem} 
Let $\kr^{(L)}$ denote CGPK-EqNet of depth $L$, filter size $q$ and ReLU activation. Then,
\begin{enumerate}
    \item $\kr^{(L)}$ can be written as a power series, $\kr^{(L)}(\tbf)=\sum_{\n \ge 0} b_\n \tbf^\n$ with
    \begin{align*}
    c_1\n^{-2.5}\leq  b_\n \leq c_2\n^{-\left(1 + \frac{3}{2d}\right)}.
    \end{align*} 
    \item The eigenvalues of $\kr^{(L)}$ are bounded by
    \begin{align*}
        c_3 \prod_{i=1}^{R}  k_i^{-(\qbar+2)}\leq \lambda_\kk\leq c_4 \prod_{i=1}^{R}   k_i^{-\left(\qbar+\frac{3}{d}-1 \right)}.
    \end{align*}
The coefficients $c_1,c_2,c_3,c_4$ are constants that depend on $\sgn(\n)$, and they equal zero if $\n$ includes non-zero values outside of the receptive field of $\kr^{(L)}$.
\end{enumerate}
\end{theorem}

\begin{theorem} 
Let $\kr^{(L)}$ denote CNTK-EqNet of depth $L$, filter size $q$ and ReLU activation. Then,
\begin{enumerate}
    \item $\kr^{(L)}$ can be written as a power series, $\kr^{(L)}(\tbf)=\sum_{\n \ge 0} b_\n \tbf^\n$ with
    \begin{align*}
    \tilde c_1\n^{-2.5}\leq  b_\n \leq \tilde c_2\n^{-\left(1 + \frac{1}{2d}\right)}.
    \end{align*} 
    \item The eigenvalues of $\kr^{(L)}$ are bounded by
    \begin{align*}
        \sum_{j=1}^d \tilde c_3 \prod_{i=j}^{R+j}  k_i^{-(\qbar+2)}\leq \lambda_\kk\leq  \sum_{j=1}^d \tilde c_4 \prod_{i=j}^{R+j}   k_i^{-\left(\qbar+\frac{1}{d}-1 \right)}.
    \end{align*}
The coefficients $\tilde c_1,\tilde c_2,\tilde c_3,\tilde c_4$ are constants that depend on $\sgn(\n)$, and they equal zero if $\n$ includes non-zero values outside of the receptive field of $\kr^{(L)}$.
\end{enumerate}
\end{theorem}

We begin by proving the lower bound for $b_{\n}$ of CGPK-EqNet.
\begin{lemma}\label{lemma:taylor_lower_bound}
Let $\kr^{(L)}$ be a CGPK-EqNet  of depth $L$, filter size $q$ with ReLU activation Then,
$\kr^{(L)}$ can be written as a power series, $\kr^{(L)}(\tbf)=\sum_{\n \ge 0} b_\n \tbf^\n$ with
\begin{align*}
    c_1(\sgn(\n))\n^{-\nu}\leq  b_\n,
    \end{align*} 
    where  $c_1(\sgn(\n))$ is constant if the receptive field of $\kr^{(L)}$ includes $\n$ and zero otherwise and $\nu=2.5$.
\end{lemma}

\begin{proof}
We prove the lemma by induction on $L$. For $L=1$ 
\begin{align*}
    \kr^{(1)}(\tbf)=\kappa_1(t_1)=\sum_{n=0}^\infty a_nt_{1}^{n},
\end{align*}
where the equality on the right provides the power series of $\kappa_1$.
Consequently, for $\n=(n,0,..,0)$, $b_{\n}=a_{n}\sim n ^{-\nu}$, and the receptive field contains only one pixel. Therefore, $c_1(\sgn(\n))$ is constant if $\n=(n,0,..,0)$ and zero otherwise. 
For $L>1$ we denote $\kappa_1(u)=\sum_{n=0}^{\infty}a_nu^{n}$ and $g(\tbf)=\kr^{L-1}(\tbf)=\sum_{\n\geq 0}\tilde b_\n\tbf^\n$ with the induction assumption that $\tilde b_\n\geq c\n^{-\nu}$. Then we have that
\begin{align*}
    \kr^{L}(\tbf)&=\kappa_1\left(\frac{1}{q}\sum_{j=0}^{q-1}g(s_j\tbf)\right)=\sum_{n=0}^\infty \frac{a_n}{q^n}\left(\sum_{j=0}^{q-1}g(s_j\tbf)\right)^n\\
    =&\sum_{n=0}^\infty \frac{a_n}{q^n}\sum_{|\kk|=n}{n \choose \kk}\prod_{i=0}^{q-1}g^{k_i}(s_i\tbf)=^{(1)}\sum_{\kk\geq 0}\frac{a_{|\kk|}}{q^{|\kk|}}{|\kk| \choose \kk}\prod_{i=0}^{q-1}\left(\sum_{\m\geq 0}\tilde b_{s_{-i}\m}\tbf^{s_{-i}\m}\right)^{k_i}:=\sum_{\n\geq 0}b_\n \tbf^\n,
\end{align*}
where $^{(1)}$ is due to the fact that $g(s_i\tbf)=\sum_{\m\geq 0} \tilde b_{\m}(s_i\tbf)^{\m}=\sum_{\m\geq 0} \tilde b_{s_{-i}\m}\tbf^{s_{-i}\m}$. Next, using a multivariate version of the Fa\'{a} di Bruno formula (see, e.g.,~\cite{schumann2019multivariate}), we have that:
\begin{align}  \label{eq:bn}
    b_\n = \sum_{\kk\geq 0}\frac{a_{|\kk|}}{q^{|\kk|}}{|\kk| \choose \kk}\sum_{\{\n_1,...,\n_q | \sum_{i=1}^{q}k_i\n_i=\n\}} \prod_{i=0}^{q-1}\hat B_{\n_i,k_i}(..,\tilde b_{s_{-i}\m},..),
\end{align}
where $\hat B_{\n,k}(\cdot)$ denote ordinary multivariate Bell polynomials defined as
\begin{align*}
    \hat B_{\n,k}(x_{\ii_1},x_{\ii_2},...) = \sum_{\bar {\cal J}_{\n,k}}
    \frac{k!}{j_{\ii_1}!,j_{\ii_2}!...} x_{\ii_1}^{j_{\ii_1}} x_{\ii_2}^{j_{\ii_2}}... 
\end{align*}
and $\bar {\cal J}_{\n,k} = \{j_{\ii_1}+j_{\ii_2}+...=k \in \Real; ~ j_{\ii_1} \ii_1 + j_{\ii_2} \ii_2 + ... =\n \in \Real^d\}$.
Since all terms in \eqref{eq:bn} are non-negative, it suffices to choose one term to get a lower bound. Specifically, we choose $\kk=(1,1..,1) \in \Real^q$ and $\n_1,\n_q$ such that $\n_1+\n_q=\n$, $\n_1^T\n_q=0$, and $\n_i=0$ for $i\notin \{1,q\}$. Noting that $|\kk|=q$, $\hat B_{\n_1,1} = \tilde b_{\n_1}$ and $\hat B_{\n_q,1} = \tilde b_{\n_q}$, and $\hat B_{\zero,1}=b_0$, we obtain
\begin{align*}
    b_{\n} \geq \frac{a_q}{q^q}q! \, \tilde b_{0}^{q-2} \tilde b_{\n_1} \tilde b_{\n_q}=C_q \tilde b_{\n_1} \tilde b_{\n_q} \geq^{(1)} C_q c^2\n^{-\nu},
\end{align*}
where $C_q=\frac{q^q}{q!} a_q \tilde b_0^{q-2}$ and $^{(1)}$ is due to the induction hypothesis.
\end{proof}

\begin{corollary}
The bound in Lemma \ref{lemma:taylor_lower_bound} holds also for CNTK-EqNet.
\end{corollary}

\begin{proof}
Let $\kr^{(L)}$ be a CNTK-EqNet. Denote by $b_\n(\kr^{(L)})$ as the power series coefficients of $\kr^{(L)}$. Then, by definition,
\begin{align*}
    \Sigma_{i,j}^{(l)}(\x,\z) &= \kappa_1\left(\frac{1}{q} \sum_{r=0}^{q-1}  \tilde\Sigma_{i+r,j+r}^{(l-1)}(\x,\z) \right)\\
    \Theta_{i,j}^{(l)}(\x,\z) &=  \frac{1}{q}\sum_{r=0}^{q-1} \left[ \kappa_0\left( \tilde\Sigma_{i+r,j+r}^{(l-1)}(\x,\z) \right) \tilde\Theta_{i+r,j+r}^{(l-1)}(\x,\z)  + \tilde\Sigma_{i+r,j+r}^{(l)}(\x,\z)\right],
\end{align*}
Since $\kappa_0$ and $\kappa_1$ have only positive power series coefficients it holds that
$b_\n(\kr^{(L)})=b_\n(\Theta_{i,i}^{(L)})\geq \frac{c_\sigma}{q}b_\n(\tilde\Sigma_{i,i}^{(L)})$. Note that $\tilde\Sigma_{i,i}^{(L)}$ is the CGPK-EqNet of $L$ layers and therefore we can apply the lower bound of Lemma \ref{lemma:taylor_lower_bound} to get $b_\n(\kr^{(L)})\geq \frac{c_\sigma}{q} c_1(\sgn(\n))\n^{-v}$.
\end{proof}

Next we give a general upper bound. We will use the following lemma:
To prove the above lemma we will use the following supporting lemma
\begin{lemma}\label{lemma:CNTK_to_NTK}
Let $\kr^{(L)}(\tbf)$ be either CGPK-EqNet or CNTK-EqNet of depth $L$ with filter size $q$. Let $K^{FC}_L(u)$ be a fully connected kernel (NTK or GPK receptively) of one variable $u$. Then, plugging $t_1=t_2..=t_i=u$ to $\kr^{(L)}(\tbf)$ gives that $\kr^{(L)}(\tbf)=K^{FC}_L(u)$, where $K^{FC}_L(u)$ denotes the corresponding CGPK or CNTK kernel of depth $L$ for a fully connected network.
\end{lemma}

\begin{proof}
We prove the lemma for CGPK. The proof for CNTK is similar. We perform induction on $L$ . For $L=1$ the claim is trivial. For $L>1$  plugging $t_1=...=t_i=u$ to $\kr^{(L)}(\tbf)$ together with the induction hypothesis gives us
\begin{align*}
    \kr^{(L)}(\tbf)&=\kappa_1\left(\frac{c_\sigma}{q}\sum_{j=0}^{q-1} \kr^{(L-1)}(s_j\tbf)\right)\\
    &=\kappa_1\left(\frac{c_\sigma}{q}\sum_{j=0}^{q-1} K^{FC}_{L-1}(u)\right)=\kappa_1(K^{FC}_{L-1}(u))=K^{FC}_L(u).
\end{align*}
\end{proof}

\begin{lemma}\label{lemma:sum_b_i}
Let $\kr^{(L)}$ be either CNTK-EqNet or CGPK-EqNet of depth $L$ with filter size $q$ and ReLU activation. Then,
$\kr^{(L)}$ can be written as a power series, $\kr^{(L)}(\tbf)=\sum_{\n \ge 0} b_\n \tbf^\n$, with, $\sum_{|\n|=k}b_{\n} = \Theta(a_k)$ where $a_k=k^{-\nu}$ with $\nu=2.5$ for CPGK and $\nu=1.5$ for CNTK.
\end{lemma}

\begin{proof} Let $\kappa(t)=\sum_{n=0}^\infty a_nt^n$. Using results by \cite{ chen2020deep} (Theorem 8) we have that $K_L^{FC}(t)=\sum_{n=0}^\infty \tilde a_nt^n$ where $K_L^{FC}(t)$ is the NTK or GPK model for a FC network and $\tilde a_n = \Theta(n^{-\nu})$ for $\nu=2.5,\nu=1.5$ for GPK and NTK respectively. Moreover, we have that 
\begin{align*}
    \kr^{(L)}(\tbf)=\sum_{\n \ge 0} b_\n \tbf^\n.
\end{align*} 
This, together with Lemma \ref{lemma:CNTK_to_NTK} and plugging $t_1=t_2=..=t_l=u$, yields 
\begin{align*}
    \kr^{(L)}(\tbf)=\sum_{\n}b_{\n}\tbf^{\n}=\sum_{\n}b_{\n}u^{|n|}=\sum_{k=0}^{\infty }u^{k}\sum_{|\n|=k}b_{\n}.
\end{align*}
The uniqueness of power series further implies 
\begin{align*}
    \sum_{|\n|=k}b_{\n}=\tilde a_k=\Theta(n^{-\nu}),
\end{align*}
which concludes the proof.
\end{proof}

Next we upper bound $b_{\n}$ (Lemma \ref{lemma:b_upper_bound}). We begin with a simple supporting lemma
\begin{lemma}\label{lemma:indices}
For any $d \geq 1$ positive (even) numbers $c_1,..,c_d\geq 1$, denote the two set of indices
\begin{align*}
    I_1&=\{(i_1,..,i_d)\in \mathbb{N_+}\times..\times\mathbb{N_+}| c_k/2\leq i_k \leq c_k\}\\
    I_2&=\{(i_1,..,i_d)\in \mathbb{N_+}\times..\times\mathbb{N_+}| (i_1+..+i_d)\in [c_1/2+..+c_d/2,c_1+..+c_d] \}.
\end{align*}
Then $I_1\subseteq I_2$.
\end{lemma}

\begin{proof}
Let $(i_1,..,i_d)\in I_1$. Then,
\begin{align*}
    c_1/2+..+c_d/2\leq i_1+..+i_d \leq c_1+..+c_d,
\end{align*}
implying that $(i_1,..,i_d)\in I_2$.
\end{proof}

\begin{lemma}\label{lemma:b_upper_bound}
Let $\kr(\tbf)=\sum_{\n \ge 0} b_\n \tbf^\n$  such that $\sum_{|\n|=n} b_{\n} = a_n \sim n^{-\nu}$ with $\nu > 1$. Then, there exists $c>0$ such that $b_{\n} \leq c \n^{-\left(\frac{\nu-1}{d}+1\right)}$. The implication for CNTK-EqNet ($\nu=1.5$) and for CGPK-EqNet ($\nu=2.5$) can appear in a separate lemma.
\end{lemma}

\begin{comment}
\begin{proof}
\ag{This is only an idea to the proof, we need to improve the details}
Let $c_1,..,c_l>0$. Recall that asymptotically we have $\sum_{i_1+..+i_l=n}b_{i_1,..,i_l}\leq a_n$, therefore
\begin{align*}
    \sum_{i_1=c_1,..i_l=c_l}^{i_1=n,..,i_l=n}b_{i_1,..,i_l}\leq \sum_{i=c_1+..+c_l}^{i=n\cdot l}a_i
\end{align*}
Assume by contradiction that $b_{i_{1},..,i_{l}} > (i_1\cdot..\cdot i_l)^{-(\frac{\nu}{l}+1)}$. Then we estimate the sum by an integral and get that for large $c_1,..,c_l,n$ up to multiplicative constants:
\begin{align*}
    \int_{c_1}^{n}&..\int_{c_l}^n(x_1\cdot..\cdot x_l)^{-(\frac{\nu}{l}+1)}dx_1..dx_l\approx \sum_{i_1=c_1,..i_l=c_l}^{i_1=n,..,i_l=n}(i_1\cdot..\cdot i_l)^{-(\frac{\nu}{l}+1)}\leq \sum_{i_1=c_1,..i_l=c_l}^{i_1=n,..,i_l=n}b_{i_1,..,i_l}\leq \sum_{i=c_1+..+c_l}^{i=n\cdot l}a_i\\
    \approx&  \int_{c_1+..+c_l}^{n\cdot l}x^{-(v+1)}dx\\
    \Rightarrow& \prod_{i=1}^l(x_i^{-\frac{\nu}{l}}|^{c_i}_{n})\leq x^{-\nu}|^{l\cdot n}_{c_1+..+c_l}\\
    \Rightarrow& \prod_{i=1}^l(c_i^{-\frac{\nu}{l}}-n^{-\frac{\nu}{l}})\leq ((c_1+..+c_l)^{-\nu}-(l\cdot n)^{-\nu })\\
    \Rightarrow& \prod_{i=1}^lc_i^{-\frac{\nu}{l}}\leq (c_1+..+c_l)^{-\nu}
\end{align*}
But this contradict the means inequality theorem that state that for every set of positive numbers $c_1,..,c_l$ it holds that $\frac{(c_1+..+c_l)}{l}\geq (c_1\cdot..\cdot c_l)^{\frac{1}{l}}$
\end{proof}
\end{comment}

\begin{proof}
 
Let $n_1,..,n_d \gg 1$ be large enough and denote by $  \bar n=\sum_{j=1}^d n_j$. Denote by $a_k=c\cdot k^{-\nu}$. By Lemma \ref{lemma:sum_b_i} we have that $\sum_{|\n|=k}b_{\n}\leq C a_k$ . Therefore,
\begin{align*}
    \sum_{k=\bar n/2}^{\bar n} \left(\sum_{|\n|=k}b_{\n}\right) = \sum_{|\n|=\bar n/2}^{\bar n}b_{\n} \leq  C \sum_{k=\bar n/2}^{\bar n}a_k.
\end{align*}
Here we can estimate the RHS using an integral and get  
\begin{align*}
    \sum_{k=\bar n/2}^{\bar n}a_k &\approx \int_{\bar n/2}^{\bar n} \frac{1}{x^{\nu}}dx %-(\nu-1) \left( \frac{1}{x^{\nu-1}}\left|_{\bar n/2}^{\bar n} \right. \right)\\
    %=&(\nu-1) \left(\left(\frac{\bar n}{2}\right)^{-(\nu-1)}-\bar n^{-(\nu-1)}\right)
    =(\nu-1)(2^{(\nu-1)}-1)\bar n^{-(\nu-1)}.
\end{align*}
on the other hand, by denoting 
\begin{align*}
    I_1&=\{\n\in \mathbb{N_+}\times..\times\mathbb{N_+}| n_j/2\leq n_j \leq n_j\}\\
    I_2&=\{\n\in \mathbb{N_+}\times..\times\mathbb{N_+}| |\n|\in [n_1/2+..+n_d/2,n_1+..+n_d]\},
\end{align*}
by Lemma \ref{lemma:indices} and because $b_{\n}\geq 0$ we have that
\begin{align*}
    \sum_{\n\in I_1}b_{\n}\leq \sum_{\n\in I_2}b_{\n}=\sum_{|\n|=\bar n/2}^{|\n|=\bar n}b_{\n}\leq  C \sum_{k=\bar n/2}^{k=\bar n}a_k
\end{align*}
Moreover $|I_1|=\frac{1}{2^d} n_1\cdot ... \cdot n_d\cdot$  and the smallest element in the sum is $\min_{\n\in I_1}\{b_{\n}\}=b_{n_1,..,n_d}$. Therefore,
\begin{align*}
    \frac{1}{2^d}n_1\cdot ..\cdot n_d b_{n_1,..,n_d}\leq\sum_{\n\in I_1}b_{\n}\leq (\nu-1)(2^{(\nu-1)}-1)\bar n^{-(\nu-1)},
\end{align*}
implying that
\begin{align*}
b_{n_1,..,n_d} \leq \frac{(\nu-1)(2^{(\nu-1)}-1)(n_1+..+n_d)^{-(\nu-1)}}{\frac{1}{2^d}(n_1\cdot ..\cdot n_d)}.
\end{align*}
Now applying the inequality of means we obtain $(n_1+..+n_d)/d\geq (n_1\cdot...\cdot n_d)^\frac{1}{d}$, and we finally get that
\begin{align*}
    b_{n_1,..,n_d} \leq d 2^d(\nu-1)(2^{(\nu-1)}-1)(n_1\cdot ..\cdot n_d)^{-\left(\frac{\nu-1}{d}+1\right)}.
\end{align*}
\end{proof}

\section{Trace and GAP kernels}

In this section we prove results presented in Section~\ref{sec:trace}. We prove Theorem \ref{thm:GAP}.

\begin{theorem}
Let $\kr$ be a multi-dot-product kernel with Mercer's decomposition as in \eqref{eq:mercer}, and let $\krtrace$ and $\krgap$ respectively be its trace and GAP versions. Then,
\begin{enumerate}
    \item $\krtrace(\x,\z)=\sum_{\kk,\jj} \lambtr_{\kk}Y_{\kk,\jj}(\x)Y_{\kk,\jj}(\z)$ with
    \begin{align}
        \lambtr_{\kk} = \frac{1}{d}\sum_{i=0}^{d-1} \lambda_{s_i\kk} 
    \end{align}
    Where $\lambda_{\kk}$ denote the eigenvalues of $\kr$.
    \item $\krgap(\x,\z)=\sum_{\kk,\jj} \lambtr_{\kk} \tilde Y_{\kk,\jj}(\x) \tilde Y_{\kk,\jj}(\z)$ with 
    \begin{align*}
        \tilde Y_{\kk,\jj}(\x) = \frac{1}{\sqrt{d}} \sum_{i=0}^{{\dbar-1}}Y_{s_i\kk,s_i\jj}(\x).
    \end{align*}
\end{enumerate}
\end{theorem}

\begin{proof}
(1) Let $\krtrace(\x,\z)$ be a trace kernel. By definition
\begin{align*}
    \krtrace(\x,\z)=\frac{1}{d}\sum_{i=0}^{d-1}\kr(s_i\x,s_i\z),
\end{align*}
where $\kr$ is a multi-dot-product kernel, with Mercer's decomposition
\begin{align*}
    \kr(\x,\z)=\sum_{\kk,\jj} \lambda_{\kk}Y_{\kk,\jj}(\x)Y_{\kk,\jj}(\z).
\end{align*}
Note that $\kr(s_i\x,s_i\z)$ has the same eigenfunctions as $\kr(\x,\z)$ with eigenvalues $\lambda_{s_i\kk}$.
So we get
\begin{align*}
    \krtrace(\x,\z)&=\frac{1}{d}\sum_{i=0}^{d-1}\kr(s_i\x,s_i\z)=\frac{1}{d}\sum_{i=0}^{d-1}\sum_{\kk,\jj} \lambda_{s_i\kk}Y_{\kk,\jj}(\x)Y_{\kk,\jj}(\z)\\
    &=\sum_{\kk,\jj}\frac{1}{d}\sum_{i=0}^{d-1} \lambda_{s_i\kk}Y_{\kk,\jj}(\x)Y_{\kk,\jj}(\z)=\sum_{\kk,\jj}Y_{\kk,\jj}(\x)Y_{\kk,\jj}(\z)\frac{1}{d}\sum_{i=0}^{d-1} \lambda_{s_i\kk}.
\end{align*}
Therefore, we have 
\begin{align*}
        \lambtr_{\kk} = \frac{1}{d}\sum_{i=0}^{d-1} \lambda_{s_i\kk}.
\end{align*}
(2) Let $\krgap(\x,\z)$ be GAP kernel. By definition we have that 
\begin{align*}
    \krgap(\x,\z)=\frac{1}{d^2}\sum_{i=0}^{\dbar-1} \sum_{j=0}^{\dbar-1} \kr(s_i\x,s_j\z)
\end{align*}
Where $\kr$ is a multi-dot-product kernel. Using Mercer's decomposition \eqref{eq:mercer}, we have 
\begin{align*}
    \kr(s_i\x,s_j\z)=\sum_{\kk,\jj} \lambda_{\kk}Y_{\kk,\jj}(s_i\x)Y_{\kk,\jj}(s_j\z)=\sum_{\kk,\jj} \lambda_{\kk}Y_{s_{-i}\kk,s_{-i}\jj}(\x)Y_{s_{-j}\kk,s_{-j}\jj}(\z).
\end{align*}
Therefore,
\begin{align*}
    \krgap(\x,\z)=&\frac{1}{d^2}\sum_{i=0}^{\dbar-1} \sum_{j=0}^{\dbar-1} \kr(s_i\x,s_j\z)=\frac{1}{d^2}\sum_{i=0}^{\dbar-1} \sum_{j=0}^{\dbar-1}\sum_{\kk,\jj} \lambda_{\kk}Y_{s_{-i}\kk,s_{-i}\jj}(\x)Y_{s_{-j}\kk,s_{-j}\jj}(\z)\\
    =&\sum_{\kk,\jj}\frac{1}{d^2}\lambda_{\kk}\sum_{i=0}^{\dbar-1} \sum_{j=0}^{\dbar-1} Y_{s_{-i}\kk,s_{-i}\jj}(\x)Y_{s_{-j}\kk,s_{-j}\jj}(\z)\\
    =&\sum_{\kk,\jj}\frac{1}{d^2}\lambda_{\kk}\left(\sum_{i=0}^{\dbar-1}Y_{s_{-i}\kk,s_{-i}\jj}(\x)\right) \left(\sum_{j=0}^{\dbar-1} Y_{s_{-j}\kk,s_{-j}\jj}(\z)\right).
\end{align*}
We can denote $\tilde Y_{\kk,\jj}(\x) = \frac{1}{\sqrt{d}} \sum_{i=0}^{\dbar-1} Y_{s_{i}\kk, s_{i}\jj}(\x)$. Note that $\tilde Y_{\kk,\jj}(\x)$ is invariant to all circular shifts of indices.   So we further denote by $\kk/S$ the set of indices $\kk$ modulu the set of circular shifts $s_0,s_1,..,s_{d-1}$ and write the last expression as
\begin{align*}
    \krgap(\x,\z)=&\sum_{\kk} \sum_{\jj}\frac{1}{d}\lambda_{\kk}\tilde Y_{\kk,\jj}(\x) \tilde Y_{\kk,\jj}(\z)=\sum_{\kk/S}\sum_{\jj/S} \left( \frac{1}{d}\sum_{i=0}^{d-1}\lambda_{s_i\kk} \right) \tilde Y_{\kk,\jj}(\x) \tilde Y_{\kk,\jj}(\z)\\
    =&\sum_{\kk/S}\sum_{\jj/S} \lambtr_{\kk}\tilde Y_{\kk,\jj}(\x) \tilde Y_{\kk,\jj}(\z).
\end{align*}
We conclude that the eigenfunctions are $\tilde Y_{\kk,\jj}(\x)$, and the eigenvalues are the same as $\lambtr$. Moreover, note that for any $\kk,\kk'$ such that $\forall i, ~\kk\neq s_i\kk'$ it holds that $\forall i~ Y_{\kk,\jj}(\x) \bot Y_{s_i\kk',\jj}(\x)$. Therefore, $\tilde Y_{\kk,\jj}(\x)\bot \tilde Y_{\kk',\jj}(\x)$, implying that $\{\tilde Y_{\kk,\jj}(\x)\}$ form an orthonormal basis.
\end{proof}

%\end{document}

\end{document}